\newif\ifdraft \draftfalse
\newif\iffull \fulltrue

\documentclass[11pt]{article}

\usepackage{algorithm}
\usepackage[noend]{algpseudocode}
\usepackage{fullpage}
\usepackage{amsmath, amssymb, amsthm}
\usepackage{mathtools}
\usepackage{color}
\usepackage{kpfonts}
\usepackage{xcolor}
\usepackage{multicol}
\usepackage{bbm}
\usepackage{hyperref}
\usepackage{mleftright}
\usepackage{graphicx}
\usepackage{subcaption}
\usepackage{cleveref}

\definecolor{DarkGreen}{rgb}{0.1,0.5,0.1}
\definecolor{DarkRed}{rgb}{0.5,0.1,0.1}
\definecolor{DarkBlue}{rgb}{0.1,0.1,0.5}
\usepackage[]{hyperref}
\hypersetup{
    unicode=false,          % non-Latin characters in Acrobat's bookmarks
    pdftoolbar=true,        % show Acrobat toolbar?
    pdfmenubar=true,        % show Acrobat menu?
    pdffitwindow=false,      % page fit to window when opened
    pdftitle={},    % title
    pdfauthor={}
    pdfsubject={},   % subject of the document
    pdfnewwindow=true,      % links in new window
    pdfkeywords={keywords}, % list of keywords
    colorlinks=true,       % false: boxed links; true: colored links
    linkcolor=DarkRed,          % color of internal links
    citecolor=DarkGreen,        % color of links to bibliography
    filecolor=DarkRed,      % color of file links
    urlcolor=DarkBlue,          % color of external links
}
\usepackage{xspace}
\usepackage{cleveref}
\usepackage{thmtools, thm-restate}
\usepackage[square,numbers]{natbib}

\usepackage{thmtools}
\usepackage{thm-restate}

\usepackage[suppress]{color-edits}
\addauthor{sw}{blue}
\addauthor{ls}{DarkGreen}

\newcommand\cP{\mathcal{P}}

\newcommand\cR{\mathbb{R}}

\newcommand{\E}{\mathop{\mathbb{E}}}
\newcommand\bmid{\mathrel{\Big|}}

\newcommand{\norm}[1]{\left\lVert#1\right\rVert}
\newcommand{\Tau}{\mathcal{T}}

\newcommand{\Lagr}{\mathcal{L}}

\newcommand{\ind}{\mathbbm{1}}
\newcommand{\proj}{\textit{proj}}
\newcommand{\err}{\textit{err}}

\def\epsilon{\varepsilon}

\DeclareMathOperator*{\argmin}{\mathrm{argmin}}
\DeclareMathOperator*{\argmax}{\mathrm{argmax}}

\newtheorem{theorem}{Theorem}[section]
\newtheorem{lemma}[theorem]{Lemma}

\newtheorem{remark}[theorem]{Remark}
\newtheorem{corollary}[theorem]{Corollary}

\newtheorem{definition}[theorem]{Definition}

\newcommand*\samethanks[1][\value{footnote}]{\footnotemark[#1]}

\title{An Algorithmic Framework for Fairness Elicitation}
\author{Christopher Jung\thanks{University of Pennsylvania}, Michael Kearns\samethanks, Seth Neel\samethanks, \\ Aaron Roth\samethanks, Logan Stapleton\thanks{University Of Minnesota}, Zhiwei Steven Wu\thanks{Carnegie Mellon University}}

\begin{document}

\maketitle

\begin{abstract}
We consider settings in which the right notion of fairness is not captured by simple mathematical
definitions (such as equality of error rates across groups), but might be more complex and nuanced and 
thus require {\em elicitation\/} from individual or collective stakeholders. We introduce a framework in
which pairs of individuals can be identified as requiring (approximately) equal treatment under a learned
model, or requiring ordered treatment such as ``applicant Alice should be at least as likely to receive
a loan as applicant Bob". We provide a provably convergent and {\em oracle efficient\/} algorithm for 
learning the most accurate model subject to the elicited fairness constraints, and prove generalization
bounds for both accuracy and fairness. This algorithm can also combine the elicited constraints with
traditional statistical fairness notions, thus ``correcting" or modifying the latter by the former. We report
preliminary findings of a behavioral study of our framework using human-subject fairness constraints
elicited on the COMPAS criminal recidivism dataset.
\end{abstract}

\section{Introduction}

The literature on algorithmic fairness has consisted largely of researchers proposing and showing how to impose technical definitions of fairness \cite{DworkHPRZ12,KC12, ZVGG17, msr_reduction, ADW19,gerry,verma2018fairness, narayanan2018translation, joseph2016fairness, HardtPNS16, ADW19, samsharad}. Because these imposed notions of fairness are described analytically, they are typically simplistic, and often have the form of equalizing simple error statistics across groups. Our starting point is the observation that:
\begin{enumerate}
\item This process cannot result in notions of fairness that do not have any simple, analytic description, and
\item This process also overlooks a more precursory problem: namely, \textit{who gets to define what is fair?}
\end{enumerate}  

It's unlikely that researchers alone are best fit for defining algorithmic fairness. Recent work identifies undue power imbalances \cite{paml} and biases \cite{kraemer2011ethicsalgo} that arise when algorithm designers and researchers are the only voices in conversations around ethical design. \citet{veale2018fairness} find that many machine learning practitioners are disconnected from the ``organisational and institutional realities, constraints and needs'' specific to the contexts in which their algorithms are applied. Researchers may not be able to propose a concise technical definition, e.g. statistical parity, to capture the nuances of fairness in any given context. Furthermore, many philosophers hold that \textit{stakeholders} who are affected by moral decisions and \textit{experts} who understand the context in which moral decisions are made will have the best judgment about which decisions are fair in that context \cite{wong2020democratizing, kraemer2011ethicsalgo}.

To this end, we aim to allow stakeholders and experts to play a central role in the process of defining algorithmic fairness. This is aligned with recent work on \textit{virtual democracy}, which propose and enact participatory methods to automate moral decision-making \cite{conitzer2018moral, noothigattu2018voting, kahng2019statistical, lee2019webuildai, freedman2020kidney}.

The way we involve stakeholders is motivated by two concerns:
\begin{enumerate}
 \item We want stakeholders to have free rein over how they may define fairness, e.g. we don't want to simply have them vote on whether existing, simple constraints like statistical parity or equalized odds is best; and 
\item We want non-technical stakeholders to be able to contribute, even if they may not understand the inner workings of a learning algorithm.
\end{enumerate}

We hold that people often cannot elucidate their conceptions of fairness; yet, they can identify specific scenarios where fairness or unfairness occurs.\footnote{This is philosophically akin to a theory of moral epistemology called \emph{moral perception}, which claims that we know moral facts (e.g. goodness or fairness) via perception, as opposed to knowing them via rules of morality (see \cite{blum1994moralperception}).} Drawing from individual notions of fairness like \citet{awareness,joseph2016fairness} that are defined in terms of pairwise comparisons, we therefore aim to elicit stakeholders conceptions of fairness by asking them to compare pairs of individuals in specific scenarios. Specifically, we ask whether it's fair that one particular individual should receive an outcome that is as desirable or better than the other.

When pointing out fairness or unfairness, this kind of pairwise ranking is natural. For example, after Serena Williams was penalized for a verbal interaction with an umpire in the 2018 U.S. Open Finals, tennis player James Blake tweeted, ``I have said worse and not gotten penalized. And I've also been given a `soft warning' by the ump where they tell you knock it off or I will have to give you a violation. [The umpire] should have at least given [Williams] that courtesy'' \cite{yaptangco2018tennis}. Here, Blake thinks that: 1) Williams should have been judged as or less severely than he would have been in a similar situation; and 2) the umpire's decision was unfair, because Williams was judged more severely.

Thus, we ask a set of stakeholders about a fixed set of pairs of individuals subject to a classification problem. For each pair of individuals $(A,B)$, we ask the stakeholder to choose from amongst a set of four options:

\begin{enumerate}
\item Fair outcomes must classify $A$ and $B$ the \emph{same} way (i.e. they must either both get a favorable classification or both get an unfavorable classification).
\item Fair outcomes must give $A$ an outcome that is equal to \emph{or preferable to} the outcome of $B$. 
\item Fair outcomes must give $B$ an outcome that is equal to \emph{or preferable to} the outcome of $A$
\item Fair outcomes may treat $A$ and $B$ differently without any constraints. 
\end{enumerate}
These constraints, a data distribution, and a hypothesis class define a learning problem: minimize classification error subject to the constraint that the rate of violation of the elicited pairwise constraints is held below some fixed threshold.
Crucially and intentionally we elicit relative pairwise orderings of outcomes (e.g. $A$ and $B$ should be treated equally), but do not elicit preferences for absolute outcomes (e.g. $A$ should receive a positive outcome). This is because \emph{fairness} --- in contrast to \emph{justice} --- is often conceptualized as a measure of equality of outcomes, rather than correctness of outcomes\footnote{Sidney Morgenbesser, following the Columbia University campus protests in the 1960s, reportedly said that the police had treated him unjustly, but not unfairly. He said that he was treated unjustly because the police hit him without provocation --- but not unfairly, because the police were doing the same to everyone else as well.}. In particular, it remains the job of the learning algorithm to optimize for correctness subject to elicited fairness constraints.

We remark that the premise (and the foundation for the enormous success) of machine learning is that accurate decision making rules in complex scenarios cannot be defined with simple analytic rules, and instead are best derived directly from data. Our work can be viewed similarly, as deriving fairness constraints from data elicited from experts and stakeholders.  In this paper, we solve the computational, statistical, and conceptual issues necessary to do this, and demonstrate the effectiveness of our approach via a small behavioral study.

\subsection{Results}

\paragraph{Our Model} We model individuals as having features in $\mathcal{X}$ and binary labels, drawn from some distribution $\mathcal{P}$.
A committee of \emph{stakeholders}\footnote{Though we develop our formalism as a committee of stakeholders, note that
it permits the special case of a single subjective stakeholder, which we make use of
in our behavioral study.}
$u \in \mathcal{U}$ has preferences about whether one individual should be judged better than another individual. \lsdelete{We imagine presenting each stakeholder with a set of pairs of individuals and asking them to choose one of four options for each pair, e.g. given the features of Alice and Bob:
\begin{enumerate}
\item No constraint;
\item Alice should be treated as well as Bob or better;
\item Bob should be treated as well as Alice or better; or
\item Alice and Bob should be treated similarly.
\end{enumerate}} 
\lsedit{We imagine presenting each stakeholder with a set of pairs of individuals and asking them to choose one of four options for each pair, e.g. given the features of Serena Williams and Jacob Blake:
\begin{enumerate}
\item No constraint;
\item Williams should be treated as well as Blake or better;
\item Blake should be treated as well as Williams or better; or
\item Williams and Blake should be treated similarly.
\end{enumerate}}

\lsedit{Here, when we refer to how an individual \textit{should be treated}, we mean the probability that an individual is given a positive label by the classifier. This may be a bit of a relaxation of these judgments, since they are not about actualized classifications, but rather the \textit{probabilities} of positive classification. For example, we may not consider it a violation of fairness preference (2) if Williams is judged worse than Blake in a specific scenario; yet, if an ump is \textit{more likely} to judge Williams worse than Blake in general, then this would violate this fairness preference.}

\lsedit{We represent these preferences abstractly as a set of ordered pairs $C_u \subseteq \mathcal{X}\times \mathcal{X}$ for each stakeholder $u$. If $(x,x') \in C_u$, this means that stakeholder $u$ believes that individual $x'$ must be treated as well as individual $x$ or better, i.e. ideally the classifier $h$ classifies such that $h(x') \ge h(x)$. This captures all possible responses above. For example, for Serena Williams $(s)$ and Jacob Blake $(b)$, if stakeholder $u$ responds: 
\begin{enumerate}
\item \textit{No constraint} $\Leftrightarrow\, (s,b)\not\in C_u$ nor $(b,s)\not\in C_u$;
\item \textit{Williams as well as Blake} $\Leftrightarrow\, (b,s) \in C_u$;
\item \textit{Blake as well as Williams} $\Leftrightarrow\, (s,b) \in C_u$; or
\item \textit{Treated similarly} $\Leftrightarrow\, (s,b) \in C_u$ and $(b,s) \in C_u$ (since if $h(b) \ge h(s)$ and $h(s) \ge h(b)$, then $h(s)=h(b)$).
\end{enumerate}}
%\cj{the sentences here are not complete?}

\lsdelete{Here, by treatment of an individual, we mean the probability that an individual is given a positive label by the classifier. We represent these preferences abstractly as a set of ordered pairs $C_u \subseteq \mathcal{X}\times \mathcal{X}$ for each stakeholder $u$. If $(x,x') \in C_u$, this means that stakeholder $u$ believes that $x'$ must be treated as well as $x$, meaning ideally, the output classifier $h$ is such that $h(x) \le h(x')$. Note that by including $(x,x')$ and $(x',x)$ in $C_u$, stakeholder $u$ can assert that $x$ and $x'$ should be treated similarly. }

We impose no structure on how stakeholders form their views  nor on the relationship between the views of different stakeholders --- i.e. the sets $\{C_u\}_{u \in \mathcal{U}}$ are allowed to be arbitrary (for example, they need not satisfy a triangle inequality), and need not be mutually consistent. We write $C = \cup_u C_u$.

\lsedit{We then formulate an optimization problem constrained by these pairwise fairness constraints. Since it is intractable to require that all constraints in $C$ be satisfied exactly, we formulate two different ``knobs'' with which we can quantitatively relax our fairness constraints.} \lsdelete{We then formulate a constrained optimization problem, that has two different ``knobs'' with which we can quantitatively relax our fairness constraint.} 

\lsedit{For $\gamma>0$ (our first knob), we say that the classification of an ordered pair of individuals $(x,x') \in C$ satisfies $\gamma$-fairness if the probability of positive classification for $x'$ plus $\gamma$ is no smaller than the probability of positive classification for $x$, i.e. $\E[h(x')]+\gamma \geq \E[h(x)]$. In this expression, the expectation is taken only over the randomness of the classifier $h$. Equivalently, a $\gamma$-fairness violation corresponds to the classification of an ordered pair of individuals $(x,x') \in C$ if the difference between these probabilities of positive classification is greater than $\gamma$, i.e. $\E[h(x) - h(x')] > \gamma$. Thus, $\gamma$ acts as a buffer on how likely it is that $x'$ be classified worse than $x$ before a fairness violation occurs. For example, if Blake $(b)$ receives a good label (i.e. no penalty) 80\% of the time and Williams $(s)$ 50\% of the time, then for $\gamma=0.1$ this constitutes a $\gamma$-fairness violation for the ordered pair $(b,s) \in C$, since $\E[h(b) - h(s)] = 0.3 \geq 0.1 = \gamma$.}

\lsdelete{We say that a $\gamma$-fairness violation corresponds to the classification of an ordered pair of individuals $(x, x') \in C$ if the difference between the probabilities of positive classification for $x'$ and $x$ is greater than $\gamma>0$, i.e. $\E[h(x') - h(x)] \geq \gamma$ or, equivalently, $\E[h(x)] +\gamma \geq E[h(x')$. In this expression, the expectation is taken only over the randomness of the classifier $h$.}
 
 We might ask that for \emph{no} pair of individuals do we have a $\gamma$-fairness violation:  $\max_{(x,x') \in C} \E[h(x) - h(x')] \leq \gamma$. On the other hand, we could ask for the weaker constraint that \emph{over a random draw of a pair of individuals}, the expected fairness violation is at most $\eta$ (our second knob): $\E_{(x,x') \sim \mathcal{P}^2}[ (h(x)-h(x'))  \cdot\mathbf{1}[(x,x') \in C]] \leq \eta$. We can also combine both relaxations to ask that the in expectation over random pairs, the ``excess'' fairness violation, on top of an allowed budget of $\gamma$, is at most $\eta$. \lsedit{For example, as above, if Blake receives a good label 80\% of the time and Williams 50\%, for $\gamma=0.1,$ the umpire classifier would pick up $0.2$ excess fairness violation for $(b,s)\in C$. In Section~\ref{sec:problem-formulation}, we weight these excess fairness violations by the proportion of stakeholders who agree with the corresponding fairness constraint and mandate their sum be less than $\eta$.} Subject to these constraints, we would like to find the distribution over classifiers that minimizes classification error: given a setting of the parameters $\gamma$ and $\eta$, this defines a benchmark with which we would like to compete.

\paragraph{Our Theoretical Results}
Even absent fairness constraints, learning to minimize 0/1 loss (even over linear classifiers) is computationally hard in the worst case (see e.g. \citealp{agnostic1,agnostic2}). Despite this, learning seems to be empirically tractable in the real world. To capture the \emph{additional} hardness of learning subject to fairness constraints, we follow several recent papers \cite{reductions,KNRW18} in aiming to develop \emph{oracle efficient} learning algorithms. Oracle efficient algorithms are assumed to have access to an \emph{oracle} (realized in experiments using a heuristic --- see the next section) that can solve weighted classification problems. Given access to such an oracle, oracle efficient algorithms must run in polynomial time.  We show that our fairness constrained learning problem is computationally no harder than unconstrained learning by giving such an oracle efficient algorithm (or reduction), and show moreover that its guarantees generalize from in-sample to out-of-sample in the usual way --- with respect to both accuracy and the frequency and magnitude of fairness violations. Our algorithm is simple and amenable to implementation, and we use it in our experimental results.

\paragraph{Our Experimental Results}
We implement our algorithm and run a set of experiments on the COMPAS recidivism prediction dataset, using fairness constraints elicited from 43 human subjects. We establish that our algorithm converges quickly (even when implemented with fast learning heuristics, rather than ``oracles''). We also explore the Pareto curves trading off error and fairness violations for different human subjects, and find empirically that there is a great deal of variability across subjects in terms of their conception of fairness, and in terms of the degree to which their expressed preferences are in conflict with accurate prediction. We find that most of the difficulty in
balancing accuracy with the elicited fairness constraints can be attributed to a small fraction of the constraints.

\subsection{Related Work}
Our work is related to existing notions of \emph{individual fairness} like \cite{awareness,joseph2016fairness} that conceptualize fairness as a set of constraints binding on pairs of individuals. In particular the notion of metric fairness proposed in \cite{awareness} is closely related, but distinct from the fairness notions we elicit in this work. In particular: 1) We allow for constraints that require that individual $A$ be treated better than or equal to individual $B$, whereas metric fairness constraints are symmetric, and only allow constraints of the form that $A$ and $B$ be treated similarly. In this sense our notion is more general.  2) We elicit binary judgements between pairs of individuals, whereas metric fairness is defined as a Lipschitz constraint on a real valued metric. In this sense our notion is more restrictive, although (we believe) easier to elicit. 

The most technically related piece of work is  \citet{RY18}, who prove similar generalization guarantees to ours for a  relaxation of metric fairness: our definition is slightly more general, and our generalization guarantee somewhat tighter, but technically the results are closely related.
Our conceptual focus and main results are quite different, however: for general learning problems, they prove worst-case hardness results, whereas we derive practical algorithms in the oracle-efficient model, and empirically
evaluate them on real user data.  The concurrent work of \citet{fairrep} makes a similar observation about guaranteeing fairness with respect to an unknown metric, although their aim is the orthogonal goal of fair representation learning.

\citet{awareness} first proposed the notion of individual metric-fairness that we take inspiration from, imagining fairness as a Lipschitz constraint on a randomized algorithm, with respect to some ``task-specific metric''. Since the original proposal, the question of where the metric should come from has been one of the primary obstacles to its adoption, and the focus of subsequent work. \citet{zemel} attempt to automatically learn a representation for the data (and hence, implicitly, a similarity metric) that causes a classifier to label an equal proportion of two protected groups as positive. \citet{compawareness} consider a group-fairness like relaxation of individual metric-fairness, asking that on average, individuals in pre-specified groups are classified with probabilities proportional to the average distance between individuals in those groups. They show how to learn such classifiers given access to an oracle which can evaluate the distance between two individuals according to the metric. Compared to our work, they assume the existence of a fairness metric which can be accessed using a quantitative oracle, and they use this metric to define a statistical rather than individual notion of fairness.  \citet{unknown} assumes access to an oracle which simply identifies fairness violations across pairs of individuals. Under the assumption that the oracle is exactly consistent with a metric in a simple linear class, \citet{unknown} gives a polynomial time algorithm to compete with the best fair policy in an online linear contextual bandits problem. In contrast to \citet{unknown}, we make essentially no assumptions at all on the structure of the ``fairness'' constraints.  \citet{I2019} studies the problem of metric learning with the goal of using only a small number of numeric valued queries, which are hard for human beings to answer, relying more on comparison queries. In contrast with \citet{I2019}, we do not attempt to learn a metric, and instead  directly learn a classifier consistent with the elicited pairwise fairness constraints.

\section{Problem Formulation}\label{sec:problem-formulation}
Let $S$ denote a set of labeled examples
$\{z_i = (x_i, y_i)\}_{i=1}^n,$ where $x_i\in \mathcal{X}$ is a
feature vector and $y_i\in\mathcal{Y}$ is a label. We will also write $S_X = \{x_i\}_{i=1}^n$
and $S_Y = \{y_i\}_{i=1}^n$. Throughout the paper, we will restrict attention to binary labels, so let $\mathcal{Y} = \{0,1\}$. Let $\mathcal{P}$ denote the unknown distribution over $\mathcal{X} \times \mathcal{Y}$. Let
$\mathcal{H}$ denote a hypothesis class containing binary classifiers $h:\mathcal{X} \to \mathcal{Y}$. We
assume that $\mathcal{H}$ contains a constant classifier (which
will imply that the ``fairness constrained'' ERM problem that we define is always feasible). We'll denote classification error of hypothesis $h$ by $\err(h, \mathcal{P}) := \Pr_{(x,y) \sim \mathcal{P}}(h(x) \neq y)$ and its empirical classification error by $\err(h, S) := \frac{1}{n}\sum_{i=1}^n \ind(h(x_i) \neq y_i)$.

We assume there is a set of one or more stakeholders $\mathcal{U}$, such that each stakeholder $u \in \mathcal{U}$ is identified with a set of
\lsedit{ordered} pairs $(x,x')$ of individuals $C_u \subseteq \mathcal{X}^2$: for each $(x,x') \in C_u$, stakeholder $u$ thinks that $x'$ should be treated as well as $x$ or better, i.e. ideally that for the learned classifier $h$, the classification $h(x') \geq h(x)$ (we will ask that this hold in expectation if the classifier is randomized, and will relax it in various ways).
For each ordered pair $(x, x')$, let $w_{x,x'}$ be the fraction of stakeholders who would like individual $x$ to be treated as well as $x'$: that is, $w_{x,x'} = \frac{|\{u | (x,x') \in C_u\}|}{|\mathcal{U}|}$. Note that if $(x,x') \in C_u$ and $(x',x) \in C_u$, then the stakeholder wants $x$ and $x'$ to be treated similarly in that ideally $h(x) = h(x')$.

In practice, we will not have direct access to the sets of \lsedit{ordered} pairs $C_u$ corresponding to the stakeholders $u$, but we may ask them whether particular \lsedit{ordered} pairs are in this set (see Section \ref{sec:experiment} for details about how we actually query human subjects). We model this by imagining that we present each stakeholder with a random set of pairs $A \subseteq [n]^2$, and for each \lsedit{ordered} pair $(x_i,x_j)$, ask if $x_j$ should not be treated worse than $x_i$; we learn the set of \lsedit{ordered} pairs in $A \cap C_u$ for each $u$. Define the empirical constraint set $\hat{C}_u = \{(x_i,x_j) \in C_u \}_{\forall (i,j) \in A}$ and $\hat{w}_{x_i x_j} = \frac{|\{u | (x,x') \in \hat{C}_u\}|}{|\mathcal{U}|}$, if $(i,j) \in A$ and 0 otherwise. We write that $\hat{C}=\cup_u\hat{C}_u$. For brevity, we will sometimes write $w_{ij}$ instead of $w_{x_i,x_j}$. Note that $\hat{w}_{ij} = w_{ij}$ for every $(i, j) \in A$.

Our goal will be to find the distribution over classifiers from
$\mathcal{H}$ that minimizes classification error, while satisfying the stakeholders' fairness preferences, captured by the constraints $C$. To do so, we'll try to find $D$, a probability distribution over $\mathcal{H}$, that minimizes the training error and satisfies the stakeholders' empirical fairness constraints, $\hat{C}$. For convenience, we denote the expected classification error of $D$ as $err(D, \cP) := \E_{h \sim D}[err(h,\cP)]$ and likewise its expected empirical classification error as $err(D, S) := \E_{h \sim D}[err(h,S)]$. We say that any distribution $D$ over classifiers satisfies $(\gamma, \eta)$-approximate subjective fairness if it is a feasible solution to the following constrained empirical risk minimization problem:
\begin{align}
  \label{min} \min_{D\in\Delta\mathcal{H}, \alpha_{ij} \geq 0} &err(D, S)\\
  \label{abs}  \text{such that } \forall(i,j)\in{[n]^2}: \quad &\E_{h\sim D}\left[h(x_i) - h(x_j)\right] \leq \alpha_{ij} + \gamma\\
  \label{sum}                                                   &\sum_{(i,j) \in [n]^2} \frac{\hat{w}_{ij} \alpha_{ij}}{|A|}  \leq \eta.
\end{align}
%\ar{Now that the $\eta$ constraint is normalized by $|A|$, aren't the weights $\hat{w}_{i,j}$ incorrectly scaled?}
%\cj{Yeah, the normalization constant should be $\sum_{(i,j) \in A} \hat{w}_{ij}$}\ar{No it shouldn't? I think the normalization constant should be 1. e.g. the problem is still well defined even if the stakeholders don't care about any pairs. Removing normalization for now.}

This ``Fair ERM'' problem, whose feasible region we denote by $\Omega(S, \hat{w}, \gamma, \eta)$, has decision variables $D$ and $\{\alpha_{ij}\}$, representing the distribution over classifiers and the ``fairness violation'' terms for each pair of training points, respectively. The parameters $\gamma$ and $\eta$ are constants which represent the two different ``knobs'' we have at our disposal to quantitatively relax the fairness constraint, in an $\ell_\infty$ and $\ell_1$ sense, respectively. 

\lsedit{
The parameter $\gamma$ defines, for any ordered pair $(x_i,x_j)$, the maximum difference between the probabilities that $x_i$ and $x_j$ receive positive labels without constituting a fairness violation. The parameter $\alpha_{ij}$ captures the ``excess fairness violation'' beyond $\gamma$ for $(x_i,x_j)$. The parameter $\eta$ upper bounds the sum of these allotted excess fairness violation terms $\alpha_{ij}$, each weighted by the proportion of judges who perceive they ought to be treated similarly $\hat{w}_{ij}$ and normalized with the total number of pairs presented $|A|$. Thus, $\eta$ bounds the expected degree of dissatisfaction of the panel of stakeholders $\mathcal{U}$, over the random choice of an ordered pair $(x_i,x_j) \in A$ and the randomness of their classification. We iterate over all $(i,j) \in [n]^2$ (not just those in $\hat{C}$) because $\hat{w}_{ij}=0$ if no judge prefers $x_i$ should be classified as well as $x_j$.}

\lsedit{To better understand $\gamma$ and $\eta$, we consider them in isolation. First, suppose we set $\gamma=0$. Then, \textit{any} difference in probabilities of positive classification between pairs is deemed a fairness violation. So, if we choose $(D,\{\alpha_{ij}\})$ such that the sum of weighted differences in positive classification probabilities exceeds $\eta$, i.e.\small\vspace{-.1cm} \[\sum_{(i,j)\in[n]^2} \frac{\hat{w}_{ij}\E_{h\sim\,D}[h(x_i)-h(x_j)]}{|A|}>\eta,\]
\normalsize then this is an infeasible solution. For example, 50\% of stakeholders think that Serena Williams $(s)$ should be treated as well as James Blake $(b)$, 70\% of stakeholders think Williams should be treated as well as John McEnroe (m), and no other constraints ($|A| = 6$); if Williams receives a good label 50\% of the time, Blake 80\%, McEnroe 90\%, and $\eta=0.07$, this is an $\eta$-fairness violation, since \small\vspace{-.1cm} 
\begin{align*}
	&\left(\hat{w}_{bs}\E[h(b)-h(s)]+\hat{w}_{ms}\E[h(m)-h(s)]\right)/|A|\\ &= \left(0.5(0.8 - 0.5) + 0.7(0.9-0.5)\right)/6 \approx 0.071 >0.07 = \eta.
\end{align*} \normalsize Second, suppose that $\eta=0$. Then, for any $(x_i,x_j)\in C$ (for which $\hat{w}_{ij}>0$), if the expected difference in labels exceeds $\gamma$, i.e. $\E_{h\sim\,D}[h(x_i)-h(x_j)]>\gamma$, then this is an infeasible solution.}

\lsdelete{To understand each of them, it helps to consider them in isolation. First, imagine that we set $\eta = 0$. $\gamma$ controls the worst-case disparity between the probability that any pair $(x_i,x_j) \in \hat C$ is classified as positive. (Note that although we have constraints for \emph{every} pair $x_i,x_j$, not just those in $\hat C$, because $\hat{w}_{i,j} = 0$ if $(x_i,x_j) \not\in \hat{C}$, a solution to the above program is free to set the slack parameter $\alpha_{i,j} = 1$ for any such pair. When $\eta = 0$, the slack parameter $\alpha_{i,j}$ is constrained to be $0$ whenever $\hat{w}_{i,j} > 0$ --- i.e. whenever $(x_i,x_j) \in \hat C$.) Next imagine that $\gamma = 0$. The parameter $\eta$ controls the \emph{expected} difference in probability that a randomly selected ordered pair $(x_i,x_j) \in A$ is classified positively, weighted by the number of stakeholders $u$ who believe that $x_i$ should be treated as well as $x_j$ --- i.e. the expected degree of dissatisfaction of the panel of stakeholders $\mathcal{U}$, over the random choice of a pair of individuals and the randomness of their classification.}
%\footnote{To see this, recall that $(x_i,x_j) \in C \Rightarrow (x_j,x_i) \in C$, and so constraint \ref{abs} can be rewritten as $|\E_{h\sim D}[h(x_i) - h(x_j)]| \leq \alpha_{ij} + \gamma$, and $\hat{w}_{i,j} = 0$ if $(i,j) \not\in A$, and so the sum in constraint \ref{sum} can equivalently be taken over $A$ rather than $[n]^2$.}.

%\ar{Are these next clarifications needed?}
%\cj{I agree. I think we can take them out}
 %The above formulation asks that the weighted average of fairness violation on each pair (with a little bit of slack $\gamma$) can't be more than $\eta$.

% Observe that by symmetry the constraint in \eqref{abs} can also be written as
%\[
%  \forall(i,j)\in{[n]^2}: \quad \left|\E_{h\sim D}[h(x_i) -
%    h(x_j)]\right| \leq \alpha_{ij} + \gamma.
%\]
%Note also that the constraint in \eqref{sum} is equivalent to
%\[
%\frac{\sum_{(i,j) \in A} w_{ij} \alpha_{ij}}{|A|} \le \eta
%\] because for $(i,j) \notin A$, $\hat{w}_{ij}=0$.

%Lastly, if we set $\eta = 0$, then the set constraints is equivalent to\cj{I think it's worth describing the difference between $\eta=0$ and $\gamma=0$ more in detail}
%\[
%  \forall(i,j)\in{[n]^2} \mbox{ such that } w_{ij} > 0: \quad
%  \E_{h\sim D}[h(x_i) - h(x_j)] \leq \gamma.
%\]

\subsection{Fairness Loss}
Our goal is to develop an algorithm that will minimize its empirical error $err(D, S)$, while satisfying the empirical fairness constraints $\hat{C}$. The standard VC dimension argument states that empirical classification error will concentrate around the true classification error: we hope to show the same kind of generalization for fairness as well. To do so, we first define fairness loss with respect to our elicited fairness preferences here.

For some fixed randomized hypothesis $D \in \Delta\mathcal{H}$ and $w$, define $\gamma$-fairness loss between an ordered pair as
\[
\Pi_{D,w,\gamma}\left(\left(x,x'\right)\right) = w_{x, x'} \max\left(0, \E_{h \sim D}\left[h(x) - h(x')\right]- \gamma \right)
\]
For a set of pairs $M \subset \mathcal{X} \times \mathcal{X}$, the $\gamma$-fairness loss of $M$ is defined to be:
\[
\Pi_{D,w,\gamma}(M) = \frac{1}{|M|}\sum_{(x,x') \in M} \Pi_{D,w,\gamma}\left(\left(x,x'\right)\right)
\]
This is the expected degree to which the difference in classification probability for a randomly selected pair exceeds the allowable budget $\gamma$, weighted by the fraction of stakeholders who think that $x'$ should be treated as well as $x$. By construction, the empirical fairness loss is bounded by $\eta$ (i.e. $\Pi_{D,w,\gamma}(M) \le \sum_{ij} \frac{\hat{w}_{ij} \alpha_{ij}}{|A|} \le \eta$), and we show in Section \ref{sec:generalization},
the empirical fairness should concentrate around the true fairness loss $\Pi_{D,w,\gamma}(\cP) := \E_{x,x' \sim \cP^2}\left[\Pi_{D,w,\gamma}(x,x')\right]$.

\subsection{Cost-sensitive Classification}
In our algorithm, we will make use of a cost-sensitive classification (CSC) oracle. An instance of CSC problem can be described by a set of costs $\{(x_i, c^0_i, c^1_i)\}_{i=1}^n$ and a hypothesis class, $\mathcal{H}$. Costs $c^0_i$ and $c^1_i$ correspond to the cost of labeling $x_i$ as 0 and 1 respectively. Invoking a CSC oracle on $\{(x_i, c^0_i, c^1_i)\}_{i=1}^n$ returns a hypothesis $h^*$ such that $h^* \in \argmin_{h \in \mathcal{H}} \sum_{i=1}^n \left(h(x_i)c_i^1 + \left(1-h(x_i)\right)c_i^0\right)$. We say that an algorithm is \emph{oracle-efficient} if it runs in polynomial time assuming access to a CSC oracle.

\section{Empirical Risk Minimization}\label{sec:erm}
In this section, we give an oracle-efficient algorithm \ref{alg:primal-best-response} for approximately solving our (in-sample) constrained empirical risk minimization problem. Details are deferred to the supplement. We prove the following theorem:

\begin{theorem}\label{thm:algorithm-guarantee}
Fix parameters $\nu, C_{\tau}, C_\lambda$ that serve to trade off running time with approximation error. There is an efficient algorithm that makes $T=\left(\frac{2C_{\lambda}\sqrt{\log(n)} +C_{\tau}}{\nu}\right)^2$ CSC oracle calls and outputs a solution $(\hat{D}, \hat{\alpha})$ with the following guarantee. The objective value is approximately optimal:
\[
err(\hat{D}, S) \le \min_{(D,\alpha) \in \Omega(S, \hat{w}, \gamma, \eta)} \err(D, S) + 2\nu.
\]
 And the constraints are approximately satisfied: $\E_{h\sim \hat{D}}[h(x_i) - h(x_j)] \le \hat{\alpha}_{ij} + \gamma + \frac{1+2\nu}{C_\lambda}, \forall (i,j) \in [n]^2$ and
$\frac{1}{|A|}\sum_{(i, j)\in [n]^2} \hat{w}_{ij} \hat{\alpha}_{ij}  \le \eta + \frac{1+2\nu}{C_\tau}.$ 
\end{theorem}
%In this section, we provide our result that says the in-sample constrained optimization can be solved approximately (i.e. the objective $\err(D,S)$ will be close to optimal and the constraints will be violated only by a little) and provide a sketch of our solution.

\subsection{Outline of the Solution}
We frame the problem of solving our constrained ERM problem (equations \eqref{min} through \eqref{sum}) as finding an approximate equilibrium of a zero-sum game between a primal player and a dual player, trying to minimize and maximize respectively the Lagrangian of the constrained optimization problem.

The Lagrangian for our optimization problem is
\begin{align*}
  \Lagr(D, \alpha, \lambda, \tau) = \err(D,S) &+\sum_{(i,j) \in [n]^2} \lambda_{ij} \left( \E_{h\sim D}[h(x_i) - h(x_j)] - \alpha_{ij} -\gamma \right) \\
  &+\tau \left( \frac{1}{|A|}\sum_{(i, j)\in [n]^2} w_{ij} \alpha_{ij} - \eta \right)
\end{align*}

For the constraint in equation \eqref{abs}, corresponding to the $\gamma$-fairness violation for each ordered pair of individuals $(x_i,x_j)$, we introduce a dual variable $\lambda_{ij}$. For the constraint \eqref{sum}, which corresponds to the $\eta$-fairness violation over all pairs of individuals, we introduce a dual variable of $\tau$. \lsedit{For brevity, we define vectors $\lambda \in \Lambda$ and $\alpha$ which are made up of all the multipliers $\lambda_{ij}$ and the excess fairness violation allotments $\alpha_{ij}$, respectively.} The primary player's action space is $(D, \alpha) \in (\Delta\mathcal{H}, [0,1]^{n^2})$, and the dual player's action space is $(\lambda, \tau) \in (\cR^{n^2}, \cR)$.%, which corresponds to how much he wants to charge each constraint for the violation.

Solving our constrained ERM problem is equivalent to finding a minmax equilibrium of $\Lagr$:
\begin{align*}
&\argmin_{(D,\alpha) \in \Omega(S, \hat{w}, \gamma, \eta)} \err(D, S) = \argmin_{D \in \Delta\mathcal{H}, \alpha \in [0,1]^{n^2}} \max_{\lambda \in \cR^{n^2}, \tau \in \cR} \Lagr(D, \alpha, \lambda, \tau)
\end{align*}

Because $\Lagr$ is linear in terms of its parameters, Sion's minimax theorem \citep{sion1958general} gives us
\begin{align*}
\hspace{-.7cm}\min_{D \in \Delta\mathcal{H}, \alpha \in [0,1]^{n^2}} \max_{\lambda \in \cR^{n^2}, \tau \in \cR} \hspace{-.3cm} \Lagr(D, \alpha, \lambda, \tau) 
=   \hspace{-.2cm} \max_{\lambda \in \cR^{n^2}, \tau \in \cR} \min_{D \in \Delta\mathcal{H}, \alpha \in [0,1]^{n^2}} \hspace{-.4cm} \Lagr(D, \alpha, \lambda, \tau).
\end{align*}

By a classic result of \citet{freund1996game}, one can compute an approximate equilibrium by simulating ``no-regret'' dynamics between the primal and dual player. \lsedit{``No-regret'' meaning that the average \textit{regret} --or difference between our algorithm's plays and the single best play in hindsight-- is bounded above by a term that converges to zero with increasing rounds.}

\lsedit{In our case, we define a zero-sum game wherein the primary player's plays from action space $(D, \alpha) \in (\Delta\mathcal{H}, [0,1]^{n^2})$, and the dual player's plays from action space $(\lambda, \tau) \in (\cR_{\geq0}^{n^2}, \cR_{\geq0})$. In any given round $t$, the dual player plays first and the primal second. The primal player can simply best respond to the dual player (see Algorithm~\ref{alg:primal-best-response}).}

\lsedit{However, since the dual player plays first, they cannot simply best respond to the primal player's action. The dual player has to anticipate the primal player's best response in order to figure out what to play. Ideally, the dual player would enumerate every possible primal play and calculate the best dual response. However, this is intractable. So, the dual player updates dual variables $\{\lambda,\tau\}$ according to \textit{no-regret} learning algorithms (exponentiated gradient descent \cite{kivinenwarmuth1997exponentiated} and online gradient descent \cite{zinkevich2003online}, respectively).}

\lsedit{The time-averaged play of both players converges to an approximate equilibrium of the zero-sum game, where the approximation is controlled by the regret of the dual player. This approximate equilibrium corresponds to an approximate saddle point for the Lagrangian $\Lagr$, which is equivalent to an approximate solution to the Fair ERM problem.}

\lsedit{We organize the rest of this section as follows. First, for simplicity, we show how the primal player updates $\{D,\alpha\}$ (even though the dual player plays first). Second, we show how the dual player updates $\{\lambda,\tau\}$. Finally, we prove that these updates are no-regret and relate the regret of the dual player to the approximation of the solution to the Fair ERM problem.}

\lsdelete{Our algorithm can be viewed as simulating the following \textit{no-regret dynamics} between the primal and the dual players over $T$ rounds. Over each of the rounds, the dual player updates dual variables $\{ \lambda, \tau \}$ according to \textit{no-regret} learning algorithms (exponentiated gradient descent \citealp{kivinenwarmuth1997exponentiated} and online gradient descent \citealp{zinkevich2003online} respectively). At every round, the primal player then best responds with a pair $\{ D, \alpha \}$ using a CSC oracle. The time-averaged play of both players converges to an approximate equilibrium of the zero-sum game, where the approximation is controlled by the regret of the dual player. \lsedit{This approximate equilibrium corresponds to an approximate saddle point for the Lagrangian $\Lagr$, which is equivalent to an approximate solution to the Fair ERM problem.}}

\subsection{The Primal Player's Best Response}\label{subsec:primal_best_response}

In each round $t$, given the actions chosen by the dual player $(\lambda^t, \tau^t)$, the primal player needs to best respond by choosing $(D^t,\alpha^t)$ such that $(D^t,\alpha^t) \in \argmin_{D \in \Delta\mathcal{H}, \alpha \in [0,1]^{n^2}} \Lagr(D, \alpha, \lambda^t, \tau^t).$ \lsedit{In Lemma~\ref{lem:primal_separate}, we separate the optimization problem into two: one optimization over hypothesis $D$ and one over violation factor $\alpha$. In Lemma~\ref{lem:primal_D}, the primal player updates the hypothesis $D$ by leveraging a CSC oracle.}
\lsdelete{We do so by leveraging a CSC oracle.} Given $\lambda^t$, we can set the costs as follows
\begin{align*}
&c_i^{0} =  \frac{1}{n} \mathbb{E}_{h\sim D}\left[\ind(y_i \neq 0)\right] \quad
&c_i^{1} =  \frac{1}{n} \mathbb{E}_{h\sim D} \left[\ind\left(y_i \neq 1\right)\right] + (\lambda^t_{ij} - \lambda^t_{ji}).
\end{align*}

Then, $D^t = h^t = CSC\left(\{(x_i, c^0_i, c^1_i)\}_{i=1}^n\right)$ (we note that the best response is always a deterministic classifier $h^t$). 

As for $\alpha^t$, we show in Lemma~\ref{lem:primal_alpha} that the primal player sets $\alpha^t_{ij}=1$ if $\tau^t \frac{w_{ij}}{|A|} - \lambda^t_{ij} \le 0$ and 0 otherwise. We provide the pseudo-code in Algorithm~\ref{alg:primal-best-response}.

\begin{algorithm}
\caption{Best Response, $BEST_{\rho}(\lambda, \tau)$, for the primal player}
\label{alg:primal-best-response}
\begin{algorithmic}[ht!]
	\State \textbf{Input:} training examples $S=\{ x_i, y_i \}_{i=1}^n,$ $\lambda \in \Lambda$, $\tau \in \Tau$, CSC oracle $CSC$
    	\For{$i = 1, \dots, n$} 
    		\If{$y_i = 0$}
			\State Set $c^0_{i}=0$
			\State Set $c^1_{i}=\frac{1}{n} + \sum_{j \neq i} \lambda_{ij} - \lambda_{ji}$ 
		\Else
			\State Set $c^0_{i}=\frac{1}{n}$
			\State Set $c^1_{i}=\sum_{j \neq i} \lambda_{ij} - \lambda_{ji}$
		\EndIf
	\EndFor
	\State $D = CSC(S, c)$
	\For{$(i,j) \in [n]^2$}
		\State $\alpha_{ij} = \begin{cases} 
      1 : & \tau \frac{w_{ij}}{|A|} - \lambda_{ij} \leq 0 \\
      0 : & \tau \frac{w_{ij}}{|A|} - \lambda_{ij} > 0.
   \end{cases}
$
	\EndFor
	\State \textbf{Output:} $D, \alpha$
\end{algorithmic}
\end{algorithm}

\begin{lemma}
\label{lem:primal_separate} 
    For fixed $\lambda, \tau$, the best response optimization for the primal player is separable, i.e.
    \[\argmin_{D,\alpha} \Lagr(D, \alpha, \lambda, \tau) = \argmin_{D} \Lagr_{\lambda, \tau}^{\rho_1}(D) \times \argmin_{\alpha} \Lagr_{\lambda, \tau}^{\rho_2}(\alpha),\]
    
    where 
    \[\Lagr_{\lambda, \tau}^{\rho_1}(D) = \err(h,D) + \sum_{(i,j) \in [n]^2} \lambda_{ij} \E\limits_{h\sim D} \left[h(x_i) - h(x_j)\right]\] and 
    \[\Lagr_{\lambda, \tau}^{\rho_2}(\alpha) = \sum_{(i,j) \in [n]^2} \lambda_{ij} \left(-\alpha_{ij}\right) + \tau \left( \frac{1}{|A|}\sum_{(i,j) \in [n]^2} w_{ij}\alpha_{ij} \right)\]

\end{lemma}

\begin{lemma}\label{lem:primal_alpha} 
For fixed $\lambda$ and $\tau$, the output $\alpha$ from $BEST_{\rho}(\lambda, \tau)$ minimizes $\Lagr_{\lambda, \tau}^{\rho_2}$
\end{lemma}
\begin{proof}
    The optimization 

    \begin{align*}
      \argmin_{\alpha} \Lagr_{\lambda, \tau}^{\rho_2} 
      &= \argmin_{\alpha} \sum_{(i,j) \in [n]^2} \lambda_{ij} \left(-\alpha_{ij}\right) + \tau \left( \frac{1}{|A|}\sum_{(i,j) \in [n]^2} w_{ij}\alpha_{ij} \right)\\
      &= \argmin_{\alpha} \ \sum_{(i,j) \in [n]^2} -\lambda_{ij} \alpha_{ij} 
      + \sum_{(i,j) \in [n]^2} \tau \frac{w_{ij}}{|A|} \alpha_{ij}\\
      &= \argmin_{\alpha} \ \sum_{(i,j) \in [n]^2} \alpha_{ij} \left(\tau \frac{w_{ij}}{|A|} -\lambda_{ij}\right).
    \end{align*}

    Note that for any pair $(i,j) \in [n]^2$, the term $\alpha_{ij} \in [0,1]$. Thus, when the constant $\tau \frac{w_{ij}}{|A|} -\lambda_{ij} \leq 0,$ we assign $\alpha_{ij}$ as the maximum bound, $1$, in order to minimize $\Lagr_{\rho_2}$. Otherwise, when $\tau \frac{w_{ij}}{|A|} -\lambda_{ij} > 0,$ we assign $\alpha_{ij}$ as the minimum bound, 0.
\end{proof}

\begin{lemma}\label{lem:primal_D}
For fixed $\lambda$ and $\tau$, the output $D$ from $BEST_{\rho}(\lambda, \tau)$ minimizes $\Lagr_{\lambda, \tau}^{\rho_1}$
\end{lemma}
\begin{proof}
    \begin{align*}
      &\argmin_{D} \Lagr_{\lambda, \tau}^{\rho_1} \\
      &= \argmin_{D} \err(D,S) + \sum_{(i,j) \in [n]^2} \lambda_{ij} \E\limits_{h\sim D} \left[h(x_i) - h(x_j)\right]\\
      &= \argmin_{D} \frac{1}{n} \sum_{i=1}^n \mathbb{E}_{h\sim D} \left[\ind(h(x_i)\neq y_i)\right]
      + \sum_{(i,j) \in [n]^2} \lambda_{ij} \E\limits_{h\sim D} [h(x_i) - h(x_j)]\\
      &= \argmin_D \  \sum_{i=1}^n \left( \frac{1}{n} \mathbb{E}_{h\sim D} \left[\ind(h(x_i)\neq y_i)\right]
      + \sum_{j \neq i} \lambda_{ij} h(x_i)
      - \sum_{j \neq i} \lambda_{ji} h(x_i)\right)\\
      &= \argmin_D \  \sum_{i=1}^n \left(\frac{1}{n} \mathbb{E}_{h\sim D} \left[\ind(h(x_i)\neq y_i)\right]
      + \sum_{j \neq i} h(x_i)\left(\lambda_{ij} - \lambda_{ji}\right) \right).
    \end{align*}

    For each $i \in [n],$ we assign the cost \[c_i^{h(x_i)} =  \frac{1}{n} \mathbb{E}_{h\sim D} \left[\ind(h(x_i)\neq y_i)\right]
      + h(x_i)\left(\lambda_{ij} - \lambda_{ji}\right).\]
    Note that the cost depends on whether $y_i = 0$ or 1. For example, take $y_i=1$ and $h(x_i)=0$. The cost
    \begin{align*}
        c_i^{h(x_i)} = c_i^0 &= \frac{1}{n} \mathbb{E}_{h\sim D} \left[\ind(h(x_i)\neq y_i)\right]
      + \sum_{j \neq i} h(x_i)\left(\lambda_{ij} - \lambda_{ji}\right)\\
      &= \frac{1}{n} \cdot 1  + \sum_{j \neq i} 0 \cdot \left(\lambda_{ij} - \lambda_{ji}\right) 
      = \frac{1}{n}
    \end{align*}
    
\end{proof}

\subsection{The Dual Player's No-regret Updates}\label{subsec:no-regret}
%CJ: commenting this out
%\begin{theorem}
%\cj{I need the actual rates of the regret to fill this out.}
%Argue that running the algorithm $T=$ with the parameters set to be gives you some $v$-approximation solution
%\end{theorem}
%\begin{proof}
%Use the above lemmas to show
%\[
%\xi_{\psi}= \frac{C_\tau}{\sqrt{T}} + \frac{C_\lambda \log(n^2+1)}{\mu_\lambda T} + \mu_\lambda k^2 C_\lambda
%\]
%
%Set $\mu_\lambda = \frac{v}{2k^2C_\lambda}$, and $T = \max\left( \frac{4C^2_{\tau}}{v^2}, \frac{4k^2C_\lambda^2 \log(n^2+1)}{v^2}\right)$. 
%
%Because $T \ge \frac{4C^2_{\tau}}{v^2}$, $\frac{v}{2} \ge \frac{C_\tau}{\sqrt{T}}$.
%
%Because $T \ge \frac{4k^2C_\lambda^2 \log(n^2+1)}{v^2}$ and $\mu_\lambda = \frac{v}{2k^2C_\lambda}$
%\begin{align*}
%\frac{C_\lambda \log(n^2+1)}{\mu_\lambda T} + \mu_\lambda k^2 C_\lambda &= \frac{C_\lambda \log(n^2+1) \cdot 2k^2C_\lambda}{v T} + \frac{v}{4}\\
%&\ge \frac{v}{4} + \frac{v}{4} = \frac{v}{2}\\
%\end{align*}
%
%Therefore, $\xi_{\psi} \le v$
%\end{proof}

In order to reason about convergence we need to restrict the dual player's action space to lie within a bounded $\ell_1$ ball, defined by the parameters $C_\tau$ and $C_\lambda$ that appear in our theorem --- and serve to trade off running time with approximation quality:
\[
\Lambda  = \left\{ \lambda \in \cR_{+}^{n^2} : \norm{\lambda}_1 \le C_\lambda \right\}, \Tau = \left\{\tau \in \cR_{+} : \norm{\tau}_1 \le C_\tau\right\}.
\]
The dual player will use exponentiated gradient descent \citep{kivinenwarmuth1997exponentiated} to update $\lambda$ and online gradient descent \citep{zinkevich2003online} to update $\tau$, where the reward function will be defined as:
\[r_\lambda(\lambda^t) = \sum_{(i,j) \in [n]^2} \lambda^t_{ij} \left( \E_{h\sim D}\left[h(x_i) - h(x_j)\right] - \alpha_{ij} -\gamma \right)\]
and
\[r_\lambda(\tau^t) =  \tau^t \left( \frac{1}{|A|}\sum_{(i, j)\in [n]^2} w_{ij} \alpha_{ij} - \eta \right).\] We provide the pseudo-code in Algorithm \ref{alg:no-regret-dynamics} but defer some of the proofs to the supplement.

\begin{algorithm}
\caption{No-Regret Dynamics}\label{alg:no-regret}
\label{alg:no-regret-dynamics}
\begin{algorithmic}[ht!]
    \State \textbf{Input:} training examples $\{ x_i, y_i \}_{i=1}^n,$ bounds $C_\lambda$ and $C_\tau$, time horizon $T$, step sizes $\mu_\lambda$ and $\{{\mu_\tau^t}\}^{t=1}_{T}$, 
    \State Set $\theta^{0}_1 = \mathbf{0} \in \mathbb{R}^{n^2}$
    \State Set $\tau^0 = 0$
    \For{$t = 1, 2, \dots, T$} 
        \State Set $\lambda^t_{ij} = C_{\lambda} \frac{\exp{\theta^{t-1}_{ij}}}{1 + \sum_{i',j' \in [n]^2} \exp{\theta^{t-1}_{i'j'}}}$ for all pairs $(i,j) \in [n]^2$
        \State Set $\tau^t = \proj_{[0, C_\tau]}\left(\tau^{t-1} + \mu^t_{\tau} \left(\frac{1}{|A|} \sum_{i,j} w_{ij} \alpha^{t-1}_{ij} - \eta\right)\right)$
        
        \State $D^t, \alpha^t \leftarrow \textnormal{BEST}_{\rho}(\lambda^t, \tau^t)$
        
%        
%        \State $\overline{D}_t \leftarrow \frac{1}{t} \sum_{t'=1}^t D_{t'}$
%        \State $\overline{\alpha}_t \leftarrow \frac{1}{t} \sum_{t'=1}^t \alpha_{t'}, \quad \Lagr_{\max} \leftarrow \Lagr \left( \overline{D}_t, \overline{\alpha}_t,  \textnormal{BEST}_{\lambda, \tau} (\overline{D}_t, \overline{\alpha}_t) \right)$
%        \State $\overline{\lambda}_t \leftarrow \frac{1}{t} \sum_{t'=1}^t \lambda_{t'}$ 
%        \State $\overline{\tau}_t \leftarrow \frac{1}{t} \sum_{t'=1}^t \tau_{t'}, \quad \Lagr_{\min} \leftarrow \Lagr \left( \textnormal{BEST}_{D, \alpha} ( \overline{\lambda}_t, \overline{\tau}_t), \overline{\lambda}_t, \overline{\tau}_t \right)$
%        \State $\nu_t \leftarrow \max \{ \Lagr(\overline{D_t}, \overline{\alpha_t}, \overline{\lambda_t}, \overline{\tau_t}) - \Lagr_{\min}, \quad \Lagr_{\max} - \Lagr(\overline{D_t}, \overline{\alpha_t}, \overline{\lambda_t}, \overline{\tau_t} \}$
%        \State
%        \If {$\nu_t \leq \nu$}
%            \State Return $(\overline{D}_t, \overline{\alpha}_t, \overline{\lambda}_t, \overline{\tau}_t)$
%        \EndIf
	\For {$(i,j) \in [n]^2$}
		\State $\theta^{t}_{ij} = \theta^{t-1}_{ij} + \mu^{t-1}_{\lambda} \left( \E_{h\sim D^t}\left[h(x_i) - h(x_j)\right] - \alpha^t_{ij} - \gamma \right)$
	\EndFor
        \State 
    \EndFor
    \State \textbf{Output:} $\frac{1}{T}\sum_{t=1}^T D^t$
\end{algorithmic}
\end{algorithm}

\begin{lemma}
\label{lem:dual_separate}
For fixed $D$ and $\alpha$, the best response optimization for the dual player is separable, i.e.
\[\argmax_{\lambda \in \Lambda, \tau \in \Tau} \Lagr(D, \alpha, \lambda, \tau) = \argmax_{\lambda \in \Lambda} \Lagr_{D, \alpha}^{\psi_1}(\lambda) \times \argmax_{\tau \in \Tau} \Lagr_{D, \alpha}^{\psi_2}(\tau), \]
where 
\[\Lagr_{D, \alpha}^{\psi_1}(\lambda) = \sum_{(i,j) \in [n]^2} \lambda_{ij} \left( \E_{h\sim D}\left[h(x_i) - h(x_j)\right] - \alpha_{ij} -\gamma \right)\] 
and 
\[\Lagr_{D, \alpha}^{\psi_2}(\tau)=\tau \left( \frac{1}{|A|} \sum_{(i, j)\in [n]^2} w_{ij} \alpha_{ij} - \eta \right).\]
\end{lemma}

\begin{lemma}\label{lem:tau-no-regret}
Running online gradient descent for $\tau^t$, i.e. $\tau^t = \proj_{[0,C_{\tau}]}\left(\tau^{t-1} + \mu^{t-1} \cdot \nabla\Lagr_{D^t, \alpha^t}^{\psi_2}\left(\tau^{t-1}\right)\right)$, with step size $\mu^t = \frac{C_{\tau}}{\sqrt{T}}$ yields the following regret
\[\max_{\tau \in \Tau} \sum_{t=1}^T \Lagr_{D^t, \alpha^t}^{\psi_2}(\tau)  - \sum_{t=1}^T \Lagr_{D^t, \alpha^t}^{\psi_2}\left(\tau^t\right) \leq C_{\tau}\sqrt{T}.\]

\end{lemma}
\begin{proof}
First, note that $\nabla\Lagr_{D^t, \alpha^t}^{\psi_2}\left(\tau^{t-1}\right) = \frac{1}{W} \sum_{ij} w_{ij} \alpha_{ij}^{t-1} - \eta$ and
\[\tau^t = proj_{[0,C_{\tau}]} \left(\tau^{t-1} + \mu_{\tau}^t \left(\frac{1}{W} \sum_{ij} w_{ij} \alpha_{ij}^{t-1} - \eta\right)\right).\]

From \cite{zinkevich2003online}, we find that the regret of this online gradient descent (translated into the terms of our paper) is bounded as follows:

\begin{equation}\label{eq:zinkevich}
\max_{\tau \in \Tau} \sum_{t=1}^T \Lagr_{D^t, \alpha^t}^{\psi_2}(\tau)  - \sum_{t=1}^T \Lagr_{D^t, \alpha^t}^{\psi_2}\left(\tau^t\right) \leq \frac{C_{\tau}^2}{2 \mu_{\tau}^T} + \frac{\left\vert\left\vert\nabla \Lagr_{D,\alpha}^{\psi_2}\right\vert\right\vert^2}{2} \sum_{t=1}^T \mu_{\tau}^t,
\end{equation}

where the bound on our target $\tau$ term is $C_{\tau}$, the gradient of our cost function at round $t$ is $\nabla \Lagr^{\psi_2}_{D^t,\alpha^t}\left(\tau^{t-1}\right)$, and the bound $\left\vert\left\vert\nabla \Lagr_{D,\alpha}^{\psi_2}\right\vert\right\vert = \sup_{\tau \in \mathcal{T},\ t\in{[T]}} \left\vert\left\vert\nabla\Lagr_{D^t, \alpha^t}^{\psi_2}\left(\tau^{t-1}\right)\right\vert\right\vert.$ To prove the above lemma, we first need to show that this bound $\left\vert\left\vert\nabla \Lagr_{D,\alpha}^{\psi_2}\right\vert\right\vert \leq 1.$

Since $w_{ij}, \alpha_{ij}, \eta \in [0,1]$ for all pairs $(i,j)$, the Lagrangian $\frac{1}{|A|} \sum_{ij} w_{ij} \alpha_{ij} - \eta = \frac{\sum_{ij} w_{ij} \alpha_{ij}}{|A|} - \eta \leq 1.$ For all $t$, the gradient
\[\left\vert\nabla\Lagr_{D^t, \alpha^t}^{\psi_2}\left(\tau^{t-1}\right)\right\vert = \frac{\sum_{ij} w_{ij} \alpha_{ij}^{t-1}}{|A|} - \eta \leq 1.\]
Thus, \[\left\vert\nabla\Lagr_{D, \alpha}^{\psi_2}\right\vert \leq 1.\]

Note that if we define $\mu_{\tau}^t= \frac{C_{\tau}}{\sqrt{T}},$ then the summation of the step sizes is equal to
\[
\sum_{t=1}^T \mu_{\tau}^t = C_{\tau} \sqrt{T}
\]

Substituting these two results into inequality~\eqref{eq:zinkevich}, we get that the regret

\[\max_{\tau \in \Tau} \sum_{t=1}^T \Lagr_{D^t, \alpha^t}^{\psi_2}(\tau)  - \sum_{t=1}^T \Lagr_{D^t, \alpha^t}^{\psi_2}\left(\tau^t\right) \leq \frac{C_{\tau}^2}{2 \left(C_{\tau}\ /\sqrt{T}\right)} + \frac{1}{2} C_{\tau}\sqrt{T} = C_{\tau}\sqrt{T} \]

\end{proof}

\begin{lemma}\label{lem:lambda-no-regret}
Running exponentiated gradient descent for $\lambda^t$ yields the following regret:
\[
\max_{\lambda \in \Lambda} \sum_{t=1}^T \Lagr_{D^t, \alpha^t}^{\psi_1}(\lambda)  - \sum_{t=1}^T \Lagr_{D^t, \alpha^t}^{\psi_1}\left(\lambda^t\right) \leq 2C_{\lambda}  \sqrt{T \log n}.
\]
\end{lemma}

\begin{proof}
In each round, the dual player gets to charge either some $(i,j)$ constraint or no constraint at all. In other words, he is presented with $n^2+1$ options. Therefore, to account for the option of not charging any constraint, we define vector $\lambda' = \left(\lambda, 0 \right)$, where the last coordinate, which will always be $0$, corresponds to the option of not charging any constraint.

Next, we define the reward vector $\zeta^t$ for $\lambda'^t$ as
\[
\zeta^t=\left(\left(\E\limits_{h\sim D^t}\left[h(x_i) - h(x_j)\right] - \alpha^t_{ij} -\gamma\right)_{i,j \in [n]^2}, 0\right).
\]

Hence, the reward function is \[r(\lambda'^t) = \zeta^t \cdot \lambda'^t = \Lagr_{D^t, \alpha^t}^{\psi_1}\left(\lambda^t\right).\]
The gradient of the reward function is
\[ \nabla r(\lambda'^t) = \left(\left(\nabla r(\lambda^t)\right)_{i,j \in [n^2]}, 0\right) = \left(\zeta^t, 0\right)\]

Note that the L-$\infty$ norm of the gradient is bounded by 1, i.e. \[\left\vert\left\vert \nabla r(\lambda'^t)\right\vert\right\vert_{\infty} \leq 1\]
because for any $t$, each respective component of the gradient, $\E\limits_{h\sim D^t}\left[h(x_i) - h(x_j)\right] - \alpha^t_{ij} -\gamma$, is bounded by 1.\\

Here, by the regret bound of \cite{kivinenwarmuth1997exponentiated}, we obtain the following regret bound:
\begin{align*}
&\max_{\lambda \in \Lambda} \sum_{t=1}^T \Lagr_{D^t, \alpha^t}^{\psi_1}(\lambda)  - \sum_{t=1}^T \Lagr_{D^t, \alpha^t}^{\psi_1}(\lambda^t) \\
&\leq \frac{\log{n}}{\mu} + \mu \left\vert\left\vert\lambda'\right\vert\right\vert_1^2 \left\vert\left\vert\nabla r(\lambda') \right\vert\right\vert_{\infty}^2 T \\
&\le \frac{\log{n}}{\mu} + \mu C_{\lambda}^2 T.
\end{align*}

If we take $\mu = \frac{1}{C_{\lambda}} \sqrt{\frac{\log n}{T}},$ the regret is bounded as follows:

\begin{equation}
\max_{\lambda \in \Lambda} \sum_{t=1}^T \Lagr_{D^t, \alpha^t}^{\psi_1}(\lambda)  - \sum_{t=1}^T \Lagr_{D^t, \alpha^t}^{\psi_1}(\lambda^t) \leq 2C_{\lambda}  \sqrt{T \log n}.
\end{equation}

\end{proof}

\begin{remark}
If the primal learner's approximate best response satisfies
\[
\sum_{t=1}^T \Lagr\left(D^t, \alpha^t, \lambda^t, \tau^t\right) - \min_{D \in \Delta(H), \alpha \in [0,1]^{n^2}} \sum_{t=1}^T \Lagr\left(D, \alpha, \lambda^t, \tau^t\right) \le \xi_\rho T
\] along with dual player's regret of $\xi_\rho T$, then $\left(\bar{D}, \bar{\alpha}, \bar{\lambda}, \bar{\tau}\right)$ is an $\left(\xi_\rho + \xi_\psi\right)$-approximate solution 
\end{remark}

\begin{theorem}\label{thm:approx-equilibrium}
Let $\left(\hat{D}, \hat{\alpha}, \hat{\lambda}, \hat{\tau}\right)$ be a $v$-approximate solution to the Lagrangian problem. More specifically, 
\[\Lagr\left(\hat{D}, \hat{\alpha}, \hat{\lambda}, \hat{\tau}\right) \le \min_{D \in \Delta(\mathcal{H}), \alpha \in [0,1]^{n^2}} \Lagr\left(D, \alpha, \hat{\lambda}, \hat{\tau}\right) + v,\]
and 
\[\Lagr(\hat{D}, \hat{\alpha}, \hat{\lambda}, \hat{\tau}) \ge \max_{\lambda \in \Lambda, \tau \in \Tau} \Lagr\left(\hat{D}, \hat{\alpha}, \lambda, \tau\right) - v.\]

Then, $err\left(\hat{D}, S\right) \le OPT + 2v$. And as for the constraints, we have
\[\E_{h\sim \hat{D}}\left[h(x_i) - h(x_j)\right] \le \hat{\alpha}_{ij} + \gamma + \frac{1+2v}{C_\lambda}, \forall (i,j) \in [n]^2\] 
\[\frac{1}{|A|}\sum_{(i, j)\in [n]^2} \hat{w}_{ij} \hat{\alpha}_{ij}  \le \eta + \frac{1+2v}{C_\tau}.\]
\end{theorem}
\begin{proof}
Let $(D^*, \alpha^*)=\argmin_{(D,\alpha) \in \Omega(S, \hat{w}, \gamma, \eta)} \err(D, S)$, the optimal solution to the Fair ERM. Also, define
\begin{align*}
&penalty_{S,w}\left(D, \alpha, \lambda, \tau\right) \\
&:= \sum_{(i,j)} \lambda_{ij} \left( \E\limits_{h\sim D}\left[h(x_i) - h(x_j)\right] - \alpha_{ij} -\gamma \right) + \tau \left( \frac{1}{|A|}\sum_{(i,j)} \hat{w}_{ij} \alpha_{ij} - \eta  \right).
\end{align*}

Note that for any $D$ and $\alpha$, $\max_{\lambda \in \Lambda, \tau \in \Tau} penalty_{S, \hat{w}}(D, \alpha, \lambda, \tau) \ge 0$ because one can always set $\lambda = 0$ and $\tau = 0$.

\begin{align*}
\max_{\lambda \in \Lambda, \tau \in \Tau} \Lagr\left(\hat{D}, \hat{\alpha}, \lambda, \tau\right) &\le \Lagr\left(\hat{D}, \hat{\alpha}, \hat{\lambda}, \hat{\tau}\right) + v\\
&\le \min_{D \in \Delta(\mathcal{H}), \alpha \in [0,1]^{n^2}} \Lagr\left(D, \alpha \hat{\lambda}, \hat{\tau}\right) + 2v\\
&\le \Lagr\left(D^*, \alpha^*, \hat{\lambda}, \hat{\tau}\right) + 2v\\
&= \err\left(D^*, S\right) + penalty_{S, \hat{w}}\left(D^*, \alpha^*, \hat{\lambda}, \hat{\tau}\right) + 2v\\
&\le \err\left(D^*, S\right) + 2v
\end{align*}

The first inequality and the third inequality are from the definition of $v$-approximate saddle point, and the second to last equality comes from the fact that $(D^*,a^*)$ is a feasible solution.

Now, we consider two cases when$(\hat{D},\hat{\alpha})$ is a feasible solution and when it's not.
\begin{enumerate}
\item $\left(\hat{D},\hat{\alpha}\right) \in \Omega\left(S,\hat{w}, \gamma, \eta\right)$\\
In this case, $\max_{\lambda \in \Lambda, \tau \in \Tau} penalty_{S, \hat{w}}\left(\hat{D}, \hat{\alpha}, \lambda, \tau\right) = 0$ because by the definition of being a feasible solution, we have $\E_{h\sim D}\left[h(x_i) - h(x_j)\right] \le \alpha_{ij} + \gamma, \forall (i,j) \in [n]^2$ and\\
$\frac{1}{|A|}\sum_{(i, j)\in [n]^2} \hat{w}_{ij} \alpha_{ij}  \le \eta.$
 Hence, $\max_{\lambda \in \Lambda, \tau \in \Tau} \Lagr\left(\hat{D}, \hat{\alpha}, \lambda, \tau\right) = err\left(\hat{D}, S\right)$. Therefore, we have $err\left(\hat{D}, S\right) \le  err\left(D^*, S\right) + 2v$.

\item $\left(\hat{D},\hat{\alpha}\right) \notin \Omega\left(S,\hat{w}, \gamma, \eta\right)$\\
\begin{align*}
&\max_{\lambda \in \Lambda, \tau \in \Tau} \Lagr\left(\hat{D}, \hat{\alpha}, \lambda, \tau\right) = err\left(\hat{D}, S\right) + \max_{\lambda \in \Lambda, \tau \in \Tau} penalty_{S, \hat{w}}\left(\hat{D}, \hat{\alpha}, \lambda, \tau\right).
\end{align*}
Therefore, $err\left(\hat{D}, S\right) \le  err\left(D^*, S\right) + 2v$ because \[\max_{\lambda \in \Lambda, \tau \in \Tau} penalty_{S, \hat{w}}\left(\hat{D}, \hat{\alpha}, \lambda, \tau\right) \ge 0.
\]

Now, we show that even when $(\hat{D}, \hat{\alpha})$ is not a feasible solution, the constraints are violated only by so much. Note that\begin{align*}
&\max_{\lambda \in \Lambda, \tau \in \Tau} \Lagr(\hat{D}, \hat{\alpha}, \lambda, \tau) \\
&= err(\hat{D}, S) + \max_{\lambda \in \Lambda, \tau \in \Tau} penalty_{S, \hat{w}}(\hat{D}, \hat{\alpha}, \lambda, \tau) \le \err(D^*, S) + 2v\\
\end{align*}
Therefore, 
\begin{align*}
&\max_{\lambda \in \Lambda, \tau \in \Tau} penalty_{S, \hat{w}}(\hat{D}, \hat{\alpha}, \hat{\lambda}, \hat{\tau}) \le \err(D^*, S)- err(\hat{D}, S) + 2v\\
&\max_{\lambda \in \Lambda, \tau \in \Tau} penalty_{S, \hat{w}}(\hat{D}, \hat{\alpha}, \hat{\lambda}, \hat{\tau}) \le 1 + 2v
\end{align*}

Let $\lambda^*, \tau^* = BEST_{\psi}\left(\hat{D}, \hat{\alpha}\right)$, which minimizes the function as shown in Lemma \ref{lem:dual_lambda} and \ref{lem:dual_tau}. Now, consider 
\[
\sum_{(i,j)} \lambda^*_{ij} \left( \E\limits_{h\sim D}\left[h(x_i) - h(x_j)\right] - \alpha_{ij} -\gamma \right) + \tau^* \left( \frac{1}{|A|}\sum_{(i,j)} \hat{w}_{ij} \alpha_{ij} - \eta  \right) \le 1 + 2v
\]

Say $(i^*,j^*) = \argmax_{(i,j) \in [n^2]} \E\limits_{h\sim D}\left[h(x_i) - h(x_j)\right] - \alpha_{ij} -\gamma$. Remember that if $\E\limits_{h\sim D}\left[h(x_{i^*})-h(x_{j^*})\right] - \alpha_{i^*j^*} -\gamma > 0$, then $\lambda^*_{i^*j^*} = C_{\tau}$ and 0 for the other coordinates and else, it's just a zero vector. Also, $\tau=C_{\tau}$ if $\sum_{(i,j)} \hat{w}_{ij} \alpha_{ij} - \eta > 0$ and 0 otherwise. Thus, 

\[
\sum_{(i,j)} \lambda^*_{ij} \left( \E\limits_{h\sim D}\left[h(x_i) - h(x_j)\right] - \alpha_{ij} -\gamma \right) \ge 0 
\] 
\[
\tau^* \left( \frac{1}{|A|}\sum_{(i,j)} \hat{w}_{ij} \alpha_{ij} - \eta  \right) \ge 0
\]

Therefore, we have
\[
\max_{i,j \in [n]^2}\left(\E\limits_{h\sim D}\left[h(x_i) - h(x_j)\right] - \alpha_{ij} -\gamma\right) \le \frac{1+2v}{C_{\lambda}},
\]
and
\[
\frac{1}{|A|}\sum_{(i, j)\in [n]^2} \hat{w}_{ij} \hat{\alpha}_{ij}  \le \eta + \frac{1+2v}{C_\tau}
\]

\end{enumerate}
\end{proof}

Now, the proof of Theorem \ref{thm:algorithm-guarantee} is simply plugging in the best response guarantee of the learner, Lemma \ref{lem:primal_alpha} and \ref{lem:primal_D}, and the no-regret guarantee of the auditor, Lemma \ref{lem:tau-no-regret} and \ref{lem:lambda-no-regret}, into Theorem \ref{thm:approx-equilibrium}. We defer the actual proof to the supplement.

%\subsection{Guarantee}
%The average of the the primal player's actions over sufficiently many rounds of no-regret dynamics converges to an approximate optimal solution.
%\begin{theorem}\label{thm:algorithm-guarantee}
%Running the algorithm for at least $T=(\frac{2C_{\lambda}\sqrt{\log(n)} +C_{\tau}}{\nu})^2$ rounds outputs $(\hat{D}, \hat{\alpha})$ with the following guarantee.

%\[err(\hat{D}, S) \le \min_{(D,\alpha) \in \Omega(S, \hat{w}, \gamma, \eta)} \err(D, S) + 2\nu.\]

% And as for the constraints, $\E_{h\sim \hat{D}}[h(x_i) - h(x_j)] \le \hat{\alpha}_{ij} + \gamma + \frac{1+2\nu}{C_\lambda}, \forall (i,j) \in [n]^2$ and
%$\frac{1}{|A|}\sum_{(i, j)\in [n]^2} \hat{w}_{ij} \hat{\alpha}_{ij}  \le \eta + \frac{1+2v}{C_\tau}.$
%\end{theorem}

\section{Generalization}\label{sec:generalization}
In this section, we show that fairness loss generalizes out-of-sample. (Error generalization follows from the standard VC-dimension bound, which --- because it is a uniform convergece statement is unaffected by the addition of fairness constraints. See the supplement for the standard statement.)

Proving that the fairness loss generalizes doesn't follow immediately from a standard VC-dimension argument for several reasons: it is not linearly separable, but defined as an average over non-disjoint \emph{pairs} of individuals in the sample. The difference between empirical fairness loss and true fairness loss of a randomized hypothesis $D \in \Delta\mathcal{H}$ is also a non-convex function of the supporting hypotheses $h$, and so it is not sufficient to prove a uniform convergence bound merely for the base hypotheses in our hypothesis class $\mathcal{H}$. We circumvent these difficulties by making use of an $\epsilon$-net argument, together with an application of a concentration inequality, and an application of Sauer's lemma. Briefly, we show that with respect to fairness loss, the continuous set of distributions over classifiers have an $\epsilon$-net of sparse distributions. Using the two-sample trick and Sauer's lemma, we can bound the number of such sparse distributions. The end result is the following generalization theorem:

%Because the error of a randomized hypothesis $D$ is just a convex combination of the error of the support hypotheses, $err(D, \cP) - err(D, S) =\E_{h \sim D}[err(h, \cP) - err(h, S)]$. The error generalization bound holds over $\Delta\mathcal{H}$ as well.

%\begin{corollary}[corollary of \cite{kearns1994introduction}]
%Fix some hypothesis class $\mathcal{H}$ and distribution $\cP$. Let $S \sim P^n$ be a dataset consisting of $n$ examples $\{x_i, y_i\}_{i=1}^n$ sampled i.i.d. from $\cP$. Then, for any $0<\delta<1$, with probability $1-\delta$, for every $D \in \Delta\mathcal{H}$, we have
%\[
%\left\vert err(D, \cP) - err(D, S)\right\vert \le O\left( \sqrt{\frac{VCDIM(\mathcal{H})+ log(\frac{1}{\delta})}{n}}\right)
%\]
%\end{corollary}

%\mk{Not really sure KV94 is the right citation for this, and the corollary above is really so standard that it might be preferable to omit it and use the
%space to provide a little interpretation of the theorem below, which right now is just dropped on the reader with no discussion.}

\begin{theorem}\label{thm:fairness-loss-generalization}
Let $S$ consists of $n$ i.i.d points drawn from $\cP$ and let $M$ represent a set of $m$ pairs randomly drawn from $S \times S$. Then we have:
\begin{align*}
&\Pr_{\substack{S \sim \cP^n \\ M \sim (S \times S)^m}}\left(\sup_{D \in \Delta\mathcal{H}}\left\vert\Pi_{D,w,\gamma}(M) - \E_{(x,x') \sim \cP^2}\left[\Pi_{D,w,\gamma}(x,x')\right]\right\vert > 2\epsilon \right) \\
&\le \left( 8 \cdot \left(\frac{e\cdot 2n}{d}\right)^{dk} \exp\left(\frac{-n\epsilon^2}{32}\right)+ \left(\frac{e\cdot 2n}{d}\right)^{dk'} \exp\left(-8m\epsilon^2\right) \right),
\end{align*}
where $k'=\frac{2 \ln(2m)}{\epsilon^2} + 1$, $k=\frac{\ln(2n^2)}{8\epsilon^2} + 1$, and $d$ is the VC-dimension of $\mathcal{H}$.
\end{theorem}
To interpret this theorem, note that the right hand side (the probability of a failure of generalization) begins decreasing exponentially fast in the data and fairness constraint sample parameters $n$ and $m$ as soon as $n \geq \Omega( d\log(n)\log(n/d))$ and $m \geq \Omega(d\log(m)\log(n/d))$.

%%% Local Variables:
%%% mode: latex
%%% TeX-master: "main"
%%% End:

\section{A Behavioral Study}
\label{sec:experiment}

%\begin{figure*}[h]
%\centering
%\includegraphics[width=0.85\textwidth]{figures/webapp2}
%\caption{Screenshot of sample subjective fairness elicitation question posed to human subjects.}
%\label{fig:webapp}
%\end{figure*}

\lsedit{The framework and algorithm we have provided can be viewed as a tool to elicit and enforce a notion of fairness defined by a collection of stakeholders. In  this section, we describe preliminary results from a human-subject study we performed in which pairwise fairness preferences were elicited and enforced by our algorithm.\\ We note that the subjects included in our empirical study were not stakeholders affected by the algorithm we used (the COMPAS algorithm). Thus, our results should not be interpreted as cogent for any policy modifications to the COMPAS algorithm. We instead report our empirical findings primarily to showcase the performance of our algorithm and to act as a template for what should be reported if our framework were applied with relevant stakeholders (for example, if fairness preferences about COMPAS data were elicited from inmates).\footnote{We omit such an empirical study due to the difficulty of accessing such stakeholders and leave this for future work.}}

\lsdelete{The framework and algorithm we have provided
can be viewed as a potentially powerful tool for empirically studying
subjective individual fairness as a {\em behavioral\/} phenomenon.\\ In this section
we describe preliminary results from a human-subject study we performed in which subjective
fairness was elicited and then enforced by our algorithm.}

\subsection{Data}

Our study used the COMPAS recidivism data gathered by ProPublica
\footnote{
The data can be accessed on ProPublica's Github
page \href{https://github.com/propublica/compas-analysis/blob/master/compas-scores-two-years.csv}{here}.
We cleaned the data as in the ProPublica study, removing any records with
missing data. This left 5829 records, where the base rate of two-year recidivism was $46\%$.
}
in their celebrated analysis of
Northepointe's risk assessment algorithm \cite{propublica1}. This data consists of defendants
from Broward County in Florida between 2013 to 2014.
For each defendant the data consists of
sex (male, female), age (18-96), race (African-American, Caucasian, Hispanic, Asian, Native American),
juvenile felony count, juvenile misdemeanor count, number of other juvenile offenses,
number of prior adult criminal offenses, the severity of the crime for which they were incarcerated (felony or misdemeanor),
as well as the outcome of whether or not they did in fact recidivate.
Recidivism is defined as a new arrest within 2 years, not counting traffic violations
and municipal ordinance violations.

\subsection{Subjective Fairness Elicitation}
\begin{figure}[h]
\centering
\includegraphics[width=1.0\textwidth]{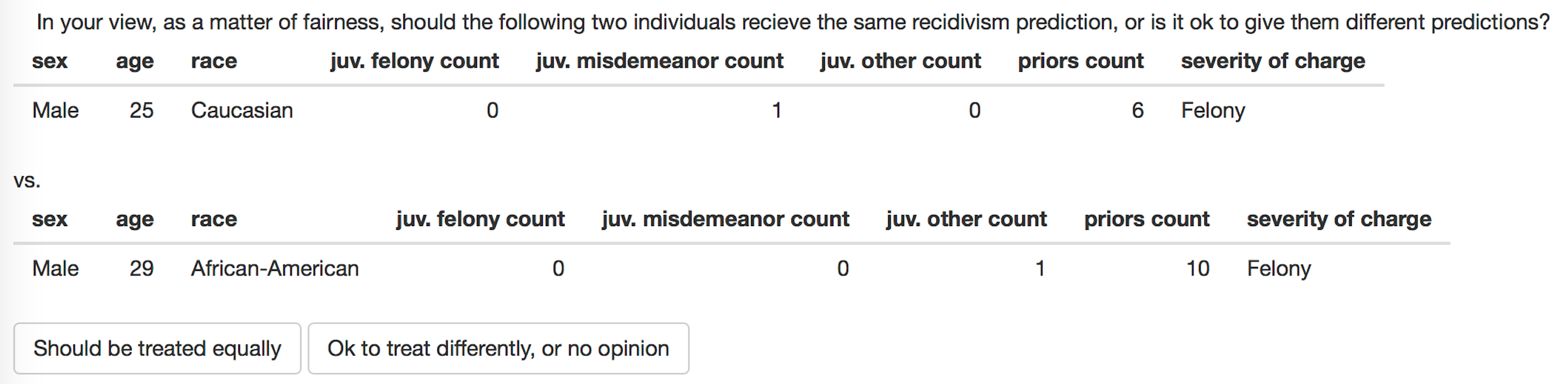}
\caption{Screenshot of sample subjective fairness elicitation question posed to human subjects.}
\label{fig:webapp}
\end{figure}

We implemented our fairness framework via a web app that
elicited subjective fairness notions from 43 undergraduates at a major research university.
After reading a document describing the data and recidivism prediction task,
each subject was presented with 50 randomly chosen pairs of records from the COMPAS data set
%as illustrated in Figure~\ref{fig:webapp}, 
and asked whether in their opinion the two individuals
should treated (predicted) equally or not. Importantly, the subjects were shown only the features for
the individuals, and not their actual recidivism outcomes, since we sought to elicit subjects' fairness notions
regarding the predictions of those outcomes. While absolutely no guidance was given to subjects regarding fairness,
the elicitation framework allows for rich possibilities. For example, subjects could choose to ignore demographic
factors or criminal histories entirely if they liked, or a subject who believes that minorities are more vulnerable
to overpolicing could discount their criminal histories relative to Caucasians in their pairwise elicitations.

For each subject, the pairs they identified to be treated equally were taken as constraints on
error minimization with respect to the actual recidivism outcomes over the entire COMPAS dataset, and our algorithm
was applied to solve this constrained optimization problem,
using a linear threshold heuristic
as the underlying learning oracle
\cite{KNRW18}.
We ran our algorithm with
$\eta = 0$ and variable $\gamma$ in Equations (\ref{min}) through (\ref{sum}), which represents the strongest
enforcement of subjective fairness --- the difference in predicted values must be at most $\gamma$ on {\em every\/} pair
selected by a subject. Because the issues we are most interested in here (convergence, tradeoffs with accuracy, and heterogeneity of fairness preferences) are orthogonal to generalization --- and because we prove VC-dimension based generalization theorems --- for simplicity, the results we report are in-sample.

\subsection{Results}

\begin{figure*}[h]

\centering
\begin{subfigure}[t]{0.3\textwidth}
	\includegraphics[width=1.0\textwidth]{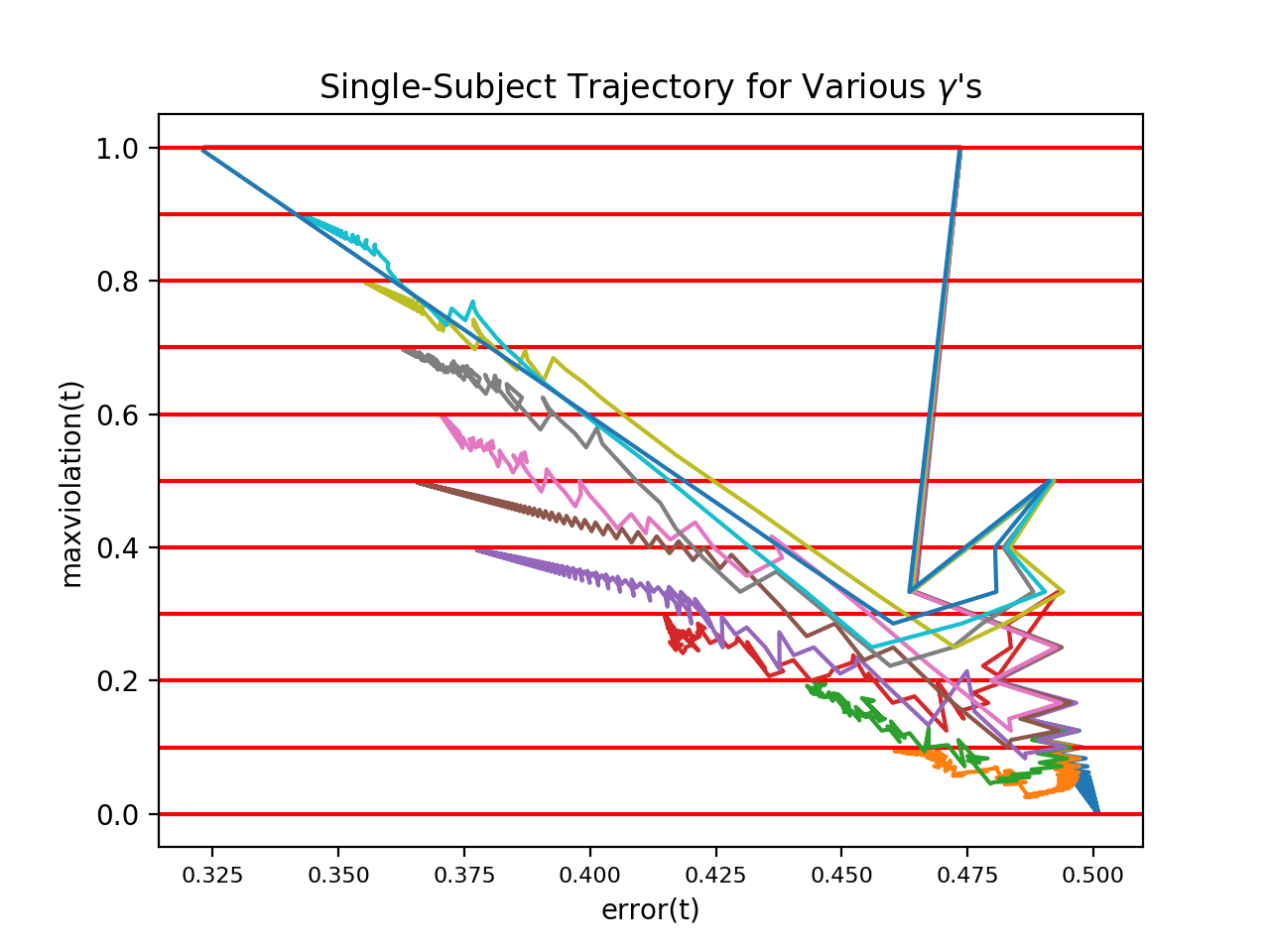}
	\subcaption{}\label{fig:trajectory}
\end{subfigure}
\begin{subfigure}[t]{0.3\textwidth}
	\includegraphics[width=1.0\textwidth]{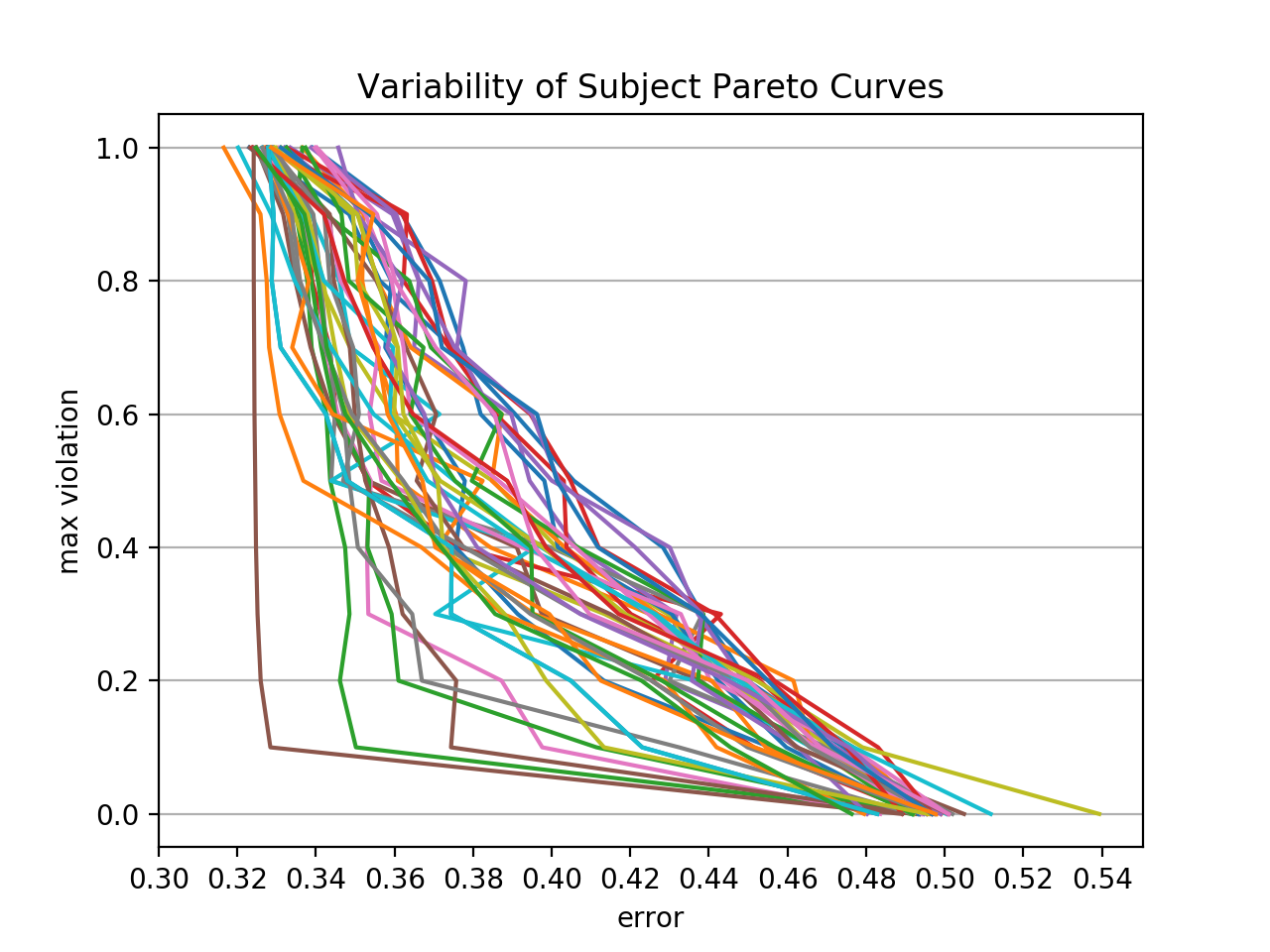}
	\subcaption{}\label{fig:pareto}
\end{subfigure}
\begin{subfigure}[t]{0.3\textwidth}
	\includegraphics[width=1.0\textwidth]{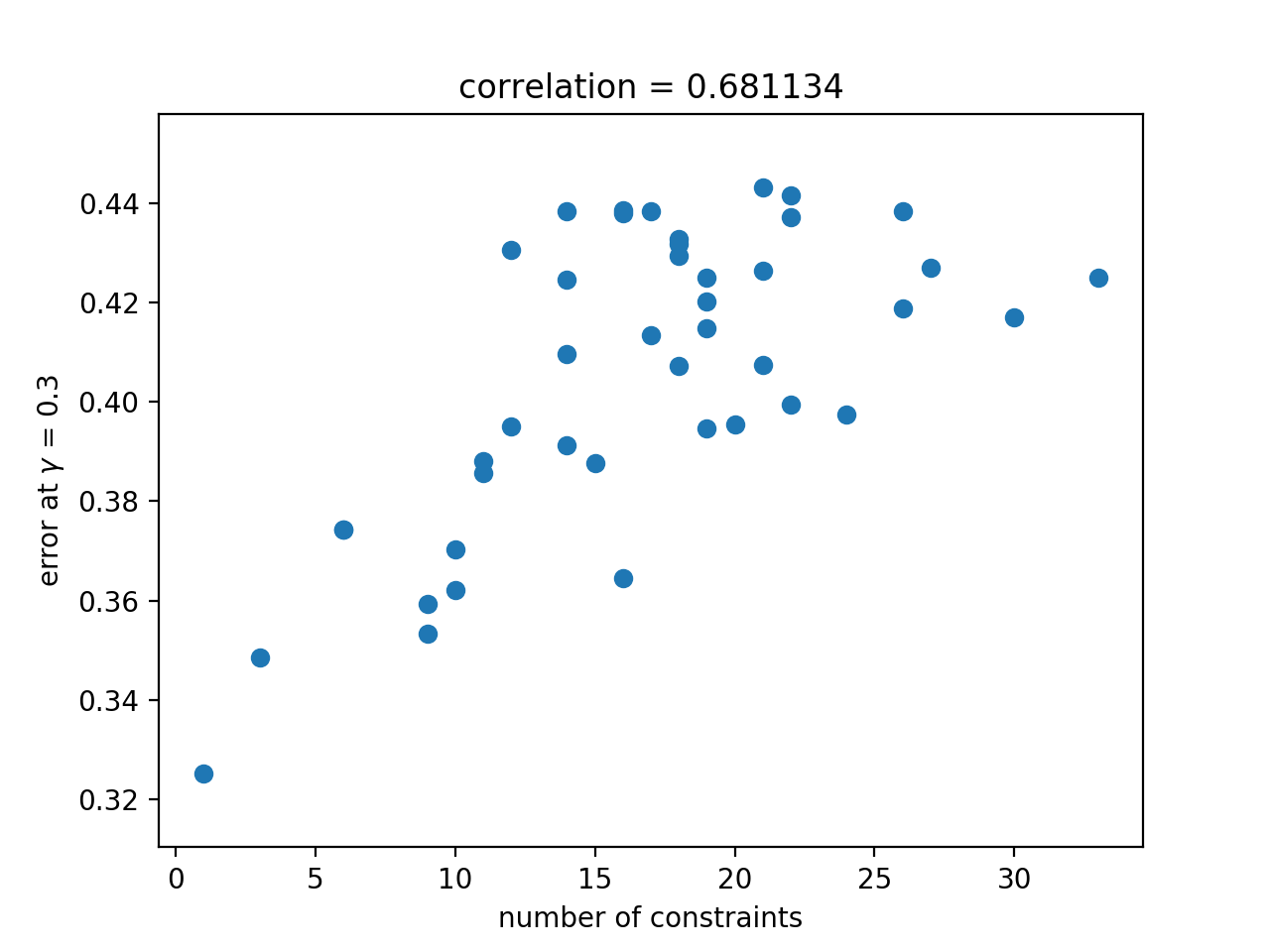}
	\subcaption{}\label{fig:numconstraints}
\end{subfigure}
\begin{subfigure}[t]{0.3\textwidth}
	\includegraphics[width=1.0\textwidth]{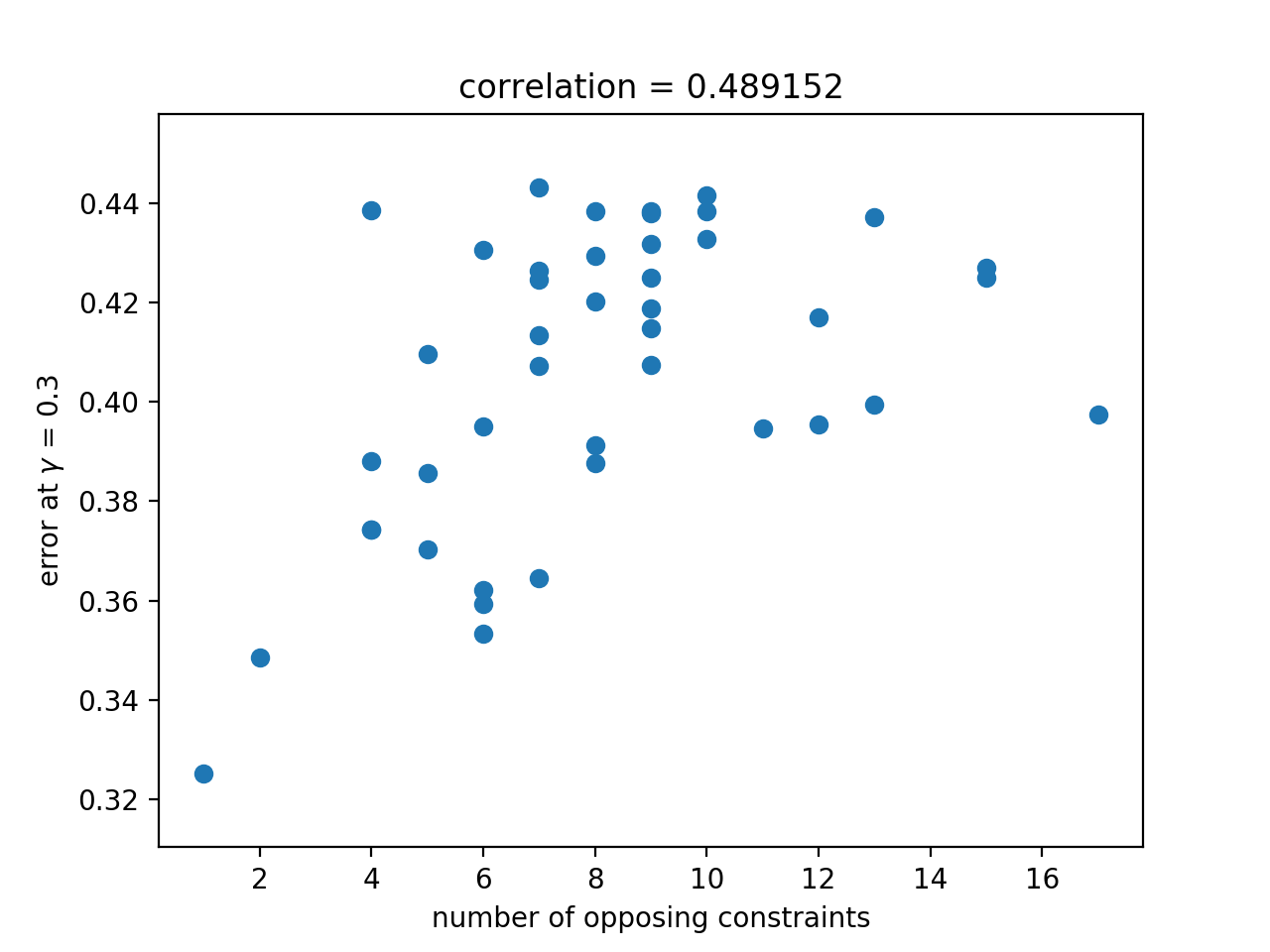}
	\caption{}
\end{subfigure}
\begin{subfigure}[t]{0.3\textwidth}
	\includegraphics[width=1.0\textwidth]{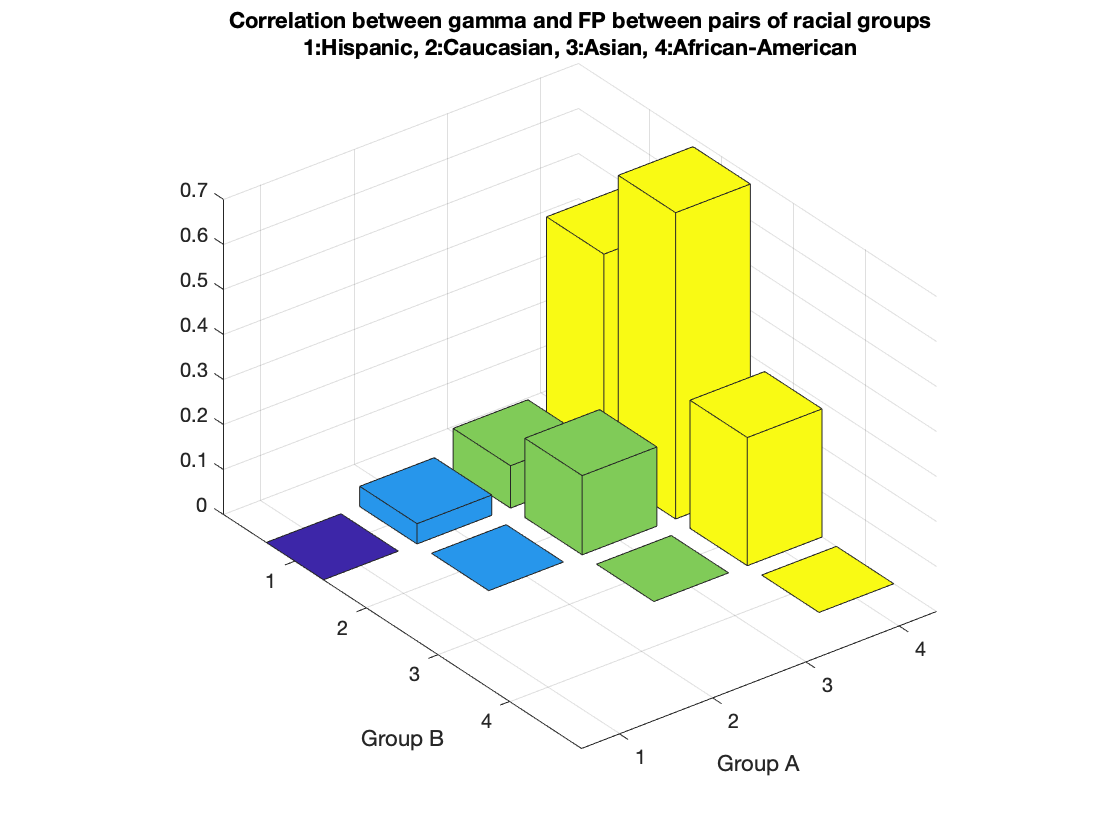}
	\subcaption{}\label{fig:correlation}
\end{subfigure}
\caption{
\label{fig:exp}
(a) Sample algorithm trajectory for a particular subject at various $\gamma$.
(b) Sample subjective fairness Pareto curves for a sample of subjects.
(c) Scatterplot of number of constraints specified and number of opposing constraints vs. error at $\gamma = 0.3$.
(d) Scatterplot of number of constraints where the true labels are different vs. error at $\gamma = 0.3$.
(e) Correlation between false positive rate difference and $\gamma$ for racial groups.
}
\end{figure*}

Since our algorithm relies on a learning heuristic for which worst-case guarantees are not possible,
the first empirical question is whether the algorithm converges rapidly on the behavioral data.
We found that it did so consistently; a typical example is Figure~\ref{fig:trajectory}, where
we show the trajectories of model error vs. fairness violation for a particular subject's data
for variable values of the input $\gamma$ (horizontal lines). After 1000 iterations, the algorithm has
converged to the optimal errors subject to the allowed $\gamma$.

Perhaps the most basic behavioral questions we might ask
involve the extent and nature of subject variability. For example,
do some subjects identify constraint pairs that are much harder to satisfy
than other subjects? And if so, what factors seem to account for such variation?

Figure~\ref{fig:pareto} shows that there is indeed considerable variation in subject difficulty.
For each of the 43 subjects, we have plotted
the error vs. fairness violation Pareto curves obtained
by varying $\gamma$ from 0 (pairs selected by subjects must have
identical probabilistic predictions of recidivism) to 1.0 (no fairness enforced whatsoever).
Since our model space is closed under probabilistic mixtures, the worst-case Pareto curve
is linear, obtained by all mixtures of the error-optimal model and random
predictions. Easier constaint sets are more convex.
We see in the figure that both extremes are exhibited behaviorally --- some subjects yield
linear or near-linear curves, while others permit huge reductions in unfairness for only
slight increases in error, and virtually all the possibilities in between are realized as well.
\footnote{The slight deviations from true convexity are due to approximate rather than exact convergence.}

Since each subject was presented with 50 random pairs and was free to constrain as many or
as few as they wished, it is natural to wonder if the variation in difficulty is explained
simply by the number of constraints chosen. In Figure~\ref{fig:numconstraints} we show a scatterplot
of the the number of constraints selected by a subject ($x$ axis)
versus the error obtained ($y$ axis) for $\gamma = 0.3$ (an intermediate value that exhibits considerable
variation in subject error rates) for all 43 subjects. While we see there is
indeed strong correlation (approximately 0.69), it is far from the case that the number of
constraints explains all the variability. For example, amongst subjects who selected approximately 16 constraints, the resulting error varies over a range of nearly 8\%, which is over 40\% of the range from the optimal error (0.32) to the worst fairness-constrained error (0.5). More surprisingly, when we consider only the `opposing' constraints, pairs of points with different true labels, the correlation (0.489) seems to be weaker. Enforcing a classifier to predict similarly on a pair of points with different true labels should increase the error, and yet, it is less correlated with error than the raw number of constraints.
This suggests that the variability in subject difficulty is due to the nature of the constraints themselves rather than their number or disagreement with the true labels.

It is also interesting to consider the collective force of the 1432 constraints selected by all
43 subjects together, which we can view as a ``fairness panel'' of sorts. Given that there are already
individual subjects whose constraints yield the worst-case Pareto curve, it is unsurprising that
the collective constraints do as well. But we can exploit the flexibility of our optimization
framework in
Equations (\ref{min}) through constraint (\ref{sum}), and let
$\gamma = 0.0$ and vary only $\eta$, thus giving the learner
 discretion in which subjects' constraints to discount or discard at a given budget
$\eta$. In doing so
we find that the unconstrained optimal error can be obtained while having the average (exact) pairwise constraint be violated by only roughly 25\%, meaning roughly that only 25\% of the collective constraints account for all the difficulty.

Finally, we can investigate the extent
to which behavioral subjective fairness notions align with
more standard statistical fairness definitions, such as equality of  false positive rates.
For instance, for each subject and a pair of racial groups, we take the absolute difference in false positive rates of the classifier at $\gamma \in \{0.0, 0.1, \dots, 1.0\}$ and calculate the correlation coefficient between realized values of $\gamma$ (which measure violation of subjective unfairness)
and the false positive rate differences. Figure~\ref{fig:correlation} shows the average correlation coefficient across subjects for each pair
of racial groups. We note that subjective fairness correlates with a smaller gap between the false positive rates across Caucasians and African Americans: but correlates substantially less for other pairs of racial groups.
% This suggests that the correlation is reflecting a trend in the fairness judgements of individuals, rather than simply being an artifact of e.g. mixing in a constant classification rule.

We leave a more complete investigation of our behavioral study for future work,
including the detailed nature of subject variability and
further comparison of behavioral subjective fairness to standard algorithmic fairness notions.

\paragraph*{Acknowledgements}
AR is supported in part by NSF grants AF-1763307, CNS-1253345, and an
Amazon Research Award.  ZSW is supported in part by an NSF grant FAI-1939606, a Google Faculty
Research Award, a J.P. Morgan Faculty Award, and a Facebook Research
Award. Part of this work was completed while ZSW was visiting the
Simons Institute for the Theory of Computing at UC Berkeley.

%\vfill
%\thispagestyle{empty}
%\setcounter{page}{0}
%\pagebreak

\bibliographystyle{plainnat}

\bibliography{./refs}

%\clearpage
%\pagenumbering{arabic}% resets `page` counter to 1
%\renewcommand*{\thepage}{A\arabic{page}}

\appendix
%\onecolumn
\section{Omitted details in Section~\ref{sec:erm}}
\subsection{Primal player's best response}

\begin{lemma}[Restatement of Lemma \ref{lem:primal_separate}]
    For fixed $\lambda, \tau$, the best response optimization for the primal player is separable, i.e.
    \[\argmin_{D,\alpha} \Lagr(D, \alpha, \lambda, \tau) = \argmin_{D} \Lagr_{\lambda, \tau}^{\rho_1}(D) \times \argmin_{\alpha} \Lagr_{\lambda, \tau}^{\rho_2}(\alpha),\]
    
    where 
    \[\Lagr_{\lambda, \tau}^{\rho_1}(D) = \err(h,D) + \sum_{(i,j) \in [n]^2} \lambda_{ij} \E\limits_{h\sim D} \left[h(x_i) - h(x_j)\right]\] and 
    \[\Lagr_{\lambda, \tau}^{\rho_2}(\alpha) = \sum_{(i,j) \in [n]^2} \lambda_{ij} \left(-\alpha_{ij}\right) + \tau \left( \frac{1}{|A|}\sum_{(i,j) \in [n]^2} w_{ij}\alpha_{ij} \right)\]

\end{lemma}
\begin{proof} First, note that $\alpha$ is not dependent on $D$ and vice versa. Thus, we may separate the optimization $\argmin_{D,\alpha} \Lagr$ as such:
    
    \begin{align*}
        &\argmin_{D,\alpha} \Lagr(D, \alpha, \lambda, \tau)\\
        &= \argmin_{D, \alpha} \err(D,S)
        + \sum_{(i,j) \in [n]^2} \lambda_{ij} \left( \E\limits_{h\sim D}\left[h(x_i) - h(x_j)\right] - \alpha_{ij} -\gamma \right) + \tau \left( \frac{1}{|A|}\sum_{(i,j) \in [n]^2} w_{ij} \alpha_{ij} - \eta \right)\\
        &= \argmin_{D} \err(D,S) + \sum_{(i,j) \in [n]^2} \lambda_{ij} \E\limits_{h\sim D} \left[h(x_i) - h(x_j)\right]
        \times \sum_{(i,j) \in [n]^2} \lambda_{ij} \left(-\alpha_{ij}\right) + \tau \left(\frac{1}{|A|}\sum_{(i,j) \in [n]^2} w_{ij}\alpha_{ij} \right)\\
        &= \argmin_{D} \Lagr_{\lambda, \tau}^{\rho_1}(D) \times \argmin_{\alpha}\Lagr_{\lambda, \tau}^{\rho_2}(\alpha)
    \end{align*}
    
\end{proof}

\subsection{Dual player's best response}

\begin{lemma}[Restatement of Lemma \ref{lem:dual_separate}]
For fixed $D$ and $\alpha$, the best response optimization for the dual player is separable, i.e.
\[\argmax_{\lambda \in \Lambda, \tau \in \Tau} \Lagr(D, \alpha, \lambda, \tau) = \argmax_{\lambda \in \Lambda} \Lagr_{D, \alpha}^{\psi_1}(\lambda) \times \argmax_{\tau \in \Tau} \Lagr_{D, \alpha}^{\psi_2}(\tau), \]
where 
\[\Lagr_{D, \alpha}^{\psi_1}(\lambda) = \sum_{(i,j) \in [n]^2} \lambda_{ij} \left( \E_{h\sim D}\left[h(x_i) - h(x_j)\right] - \alpha_{ij} -\gamma \right)\] 
and 
\[\Lagr_{D, \alpha}^{\psi_2}(\tau)=\tau \left( \frac{1}{|A|} \sum_{(i, j)\in [n]^2} w_{ij} \alpha_{ij} - \eta \right).\]
\end{lemma}
\begin{proof}
\begin{align*}
& \argmax_{\lambda \in \Lambda, \tau \in \Tau} \Lagr(D, \alpha, \lambda, \tau)\\
&=\argmax_{\lambda \in \Lambda, \tau \in \Tau} \E_{h\sim D}\left[\textit{err}(h, S)\right] + \sum_{(i,j) \in [n]^2} \lambda_{ij} \left( \E_{h\sim D}\left[h(x_i) - h(x_j)\right] - \alpha_{ij} -\gamma \right) +
  \tau \left( \frac{1}{|A|} \sum_{(i, j)\in [n]^2} w_{ij} \alpha_{ij} - \eta \right)\\
&=\argmax_{\lambda \in \Lambda} \sum_{(i,j) \in [n]^2} \lambda_{ij} \left( \E_{h\sim D}\left[h(x_i) - h(x_j)\right] - \alpha_{ij} -\gamma \right) \times \argmax_{\tau \in \Tau}
  \tau \left( \frac{1}{|A|} \sum_{(i, j)\in [n]^2} w_{ij} \alpha_{ij} - \eta \right)\\  
  &= \argmax_{\lambda \in \Lambda} \Lagr_{D, \alpha}^{\psi_1}(\lambda) \times \argmax_{\tau \in \Tau} \Lagr_{D, \alpha}^{\psi_2}(\tau)
\end{align*}
\end{proof}

\begin{algorithm}
\caption{Best Response, $BEST_{\psi}(D, \alpha)$, for the dual player}
\begin{algorithmic}[ht!]
	\State \textbf{Input:} training examples $S=\{ x_i, y_i \}_{i=1}^n,$ $D \in \Delta(H)$, $\alpha \in [0,1]^{n^2}$
    	\State $\lambda = 0 \in \mathbb{R}^{n^2}$
	\State $(i^*, j^*) = \argmax_{(i,j) \in [n]^2} \E_{h\sim D}\left[h(x_i) - h(x_j)\right] -\alpha_{ij}-\gamma$
	\If {$\E_{h\sim D}\left[h(x_{i^*}) - h(x_{j^*})\right] -\alpha_{i^*j^*}-\gamma \le 0$}
		\State $\lambda_{i^*j^*} = C_{\lambda}$
	\EndIf
	\State set $\tau = \begin{cases} 
      0 &  \frac{1}{|A|}\sum_{(i, j)\in [n]^2} w_{ij} \alpha_{ij} - \eta \le 0\\
      C_{\tau} & o.w. 
   \end{cases}$
	
	\State \textbf{Output:} $\lambda, \tau$
\end{algorithmic}
\end{algorithm}

\begin{lemma}\label{lem:dual_lambda}
For fixed $D$ and $\alpha$, the output $\lambda$ from $BEST_{\psi}(D, \alpha)$ minimizes $\Lagr_{D, \alpha}^{\psi_1}$
\end{lemma}
\begin{proof}
Because $\Lagr_{D, \alpha}^{\psi_1}$ is linear in terms of $\lambda$ and the feasible region is the non-negative orthant bounded by 1-norm, the optimal solution must include putting all the weight to the pair $(i,j)$ where $\E_{h\sim D}[h(x_i) - h(x_j)-\alpha_{ij}]$ is maximized.  
\end{proof}

\begin{lemma}\label{lem:dual_tau}
For fixed $D$ and $\alpha$, the output $\tau$ from $BEST_{\psi}(D, \alpha)$ minimizes $\Lagr_{D, \alpha}^{\psi_2}$
\end{lemma}
\begin{proof}
Because $\Lagr_{D, \alpha}^{\psi_2}$ is linear in terms of $\tau$, the optimal solution is trivially to set $\tau$ at either $C_{\tau}$ or 0 depending on the sign.
\end{proof}

\subsection{No-regret dynamics}
\begin{algorithm}
\caption{No-Regret Dynamics}\label{alg:no-regret}
\begin{algorithmic}[ht!]
    \State \textbf{Input:} training examples $\{ x_i, y_i \}_{i=1}^n,$ bounds $C_\lambda$ and $C_\tau$, time horizon $T$, step sizes $\mu_\lambda$ and $\{{\mu_\tau^t}\}^{t=1}_{T}$, 
    \State Set $\theta^{0}_1 = \mathbf{0} \in \mathbb{R}^{n^2}$
    \State Set $\tau^0 = 0$
    \For{$t = 1, 2, \dots, T$} 
        \State Set $\lambda^t_{ij} = C_{\lambda} \frac{\exp{\theta^{t-1}_{ij}}}{1 + \sum_{i',j' \in [n]^2} \exp{\theta^{t-1}_{i'j'}}}$ for all pairs $(i,j) \in [n]^2$
        \State Set $\tau^t = \proj_{[0, C_\tau]}\left(\tau^{t-1} + \mu^t_{\tau} \left(\frac{1}{|A|} \sum_{i,j} w_{ij} \alpha^{t-1}_{ij} - \eta\right)\right)$
        
        \State $D^t, \alpha^t \leftarrow \textnormal{BEST}_{\rho}(\lambda^t, \tau^t)$
        
%        
%        \State $\overline{D}_t \leftarrow \frac{1}{t} \sum_{t'=1}^t D_{t'}$
%        \State $\overline{\alpha}_t \leftarrow \frac{1}{t} \sum_{t'=1}^t \alpha_{t'}, \quad \Lagr_{\max} \leftarrow \Lagr \left( \overline{D}_t, \overline{\alpha}_t,  \textnormal{BEST}_{\lambda, \tau} (\overline{D}_t, \overline{\alpha}_t) \right)$
%        \State $\overline{\lambda}_t \leftarrow \frac{1}{t} \sum_{t'=1}^t \lambda_{t'}$ 
%        \State $\overline{\tau}_t \leftarrow \frac{1}{t} \sum_{t'=1}^t \tau_{t'}, \quad \Lagr_{\min} \leftarrow \Lagr \left( \textnormal{BEST}_{D, \alpha} ( \overline{\lambda}_t, \overline{\tau}_t), \overline{\lambda}_t, \overline{\tau}_t \right)$
%        \State $\nu_t \leftarrow \max \{ \Lagr(\overline{D_t}, \overline{\alpha_t}, \overline{\lambda_t}, \overline{\tau_t}) - \Lagr_{\min}, \quad \Lagr_{\max} - \Lagr(\overline{D_t}, \overline{\alpha_t}, \overline{\lambda_t}, \overline{\tau_t} \}$
%        \State
%        \If {$\nu_t \leq \nu$}
%            \State Return $(\overline{D}_t, \overline{\alpha}_t, \overline{\lambda}_t, \overline{\tau}_t)$
%        \EndIf
	\For {$(i,j) \in [n]^2$}
		\State $\theta^{t}_{ij} = \theta^{t-1}_{ij} + \mu^{t-1}_{\lambda} \left( \E_{h\sim D^t}\left[h(x_i) - h(x_j)\right] - \alpha^t_{ij} - \gamma \right)$
	\EndFor
        \State 
    \EndFor
    \State \textbf{Output:} $\frac{1}{T}\sum_{t=1}^T D^t$
\end{algorithmic}
\end{algorithm}

\begin{theorem}[\cite{freund1996game}]\label{thm:freund1996game}
Let $(D^1, \alpha^1), \dots, (D^T, \alpha^T)$ be the primal player's sequence of actions, and $(\lambda^1, \tau^1), \dots, (\lambda^T, \tau^T)$ be the dual player's sequence of actions. Let $\bar{D} = \frac{1}{T}\sum_{t=1}^T D^t$, $\bar{\alpha} = \frac{1}{T}\sum_{t=1}^T \alpha^t$, $\bar{\lambda} = \frac{1}{T}\sum_{t=1}^T \lambda^t$, and $\bar{\tau} = \frac{1}{T} \sum_{t=1}^T \tau^t$. Then, if the regret of the dual player satisfies 

\[
\max_{\lambda \in \Lambda, \tau \in \Tau} \sum_{t=1}^T \Lagr\left(D^t, \alpha^t, \lambda^t, \tau^t\right) - \sum_{t=1}^T \Lagr\left(D^t, \alpha^t, \lambda^t, \tau^t\right)\le \xi_\psi T,
\]
and the primal player best responds in each round ($D^t, \alpha^t = \argmax_{D \in \Delta(H), \alpha \in [0,1]^{n^2}} \Lagr\left(D, \alpha, \lambda^t, \tau^t\right)$),
then $(\bar{D}, \bar{\alpha}, \bar{\lambda}, \bar{\tau})$ is an $\xi_\psi$-approximate solution 
\end{theorem}

\iffalse
\begin{remark}
If the primal learner's approximate best response satisfies
\[
\sum_{t=1}^T \Lagr\left(D^t, \alpha^t, \lambda^t, \tau^t\right) - \min_{D \in \Delta(H), \alpha \in [0,1]^{n^2}} \sum_{t=1}^T \Lagr\left(D, \alpha, \lambda^t, \tau^t\right) \le \xi_\rho T
\] along with dual player's regret of $\xi_\rho T$, then $\left(\bar{D}, \bar{\alpha}, \bar{\lambda}, \bar{\tau}\right)$ is an $\left(\xi_\rho + \xi_\psi\right)$-approximate solution 
\end{remark}

\begin{theorem}\label{thm:approx-equilibrium}
Let $\left(\hat{D}, \hat{\alpha}, \hat{\lambda}, \hat{\tau}\right)$ be a $v$-approximate solution to the Lagrangian problem. More specifically, 
\[\Lagr\left(\hat{D}, \hat{\alpha}, \hat{\lambda}, \hat{\tau}\right) \le \min_{D \in \Delta(\mathcal{H}), \alpha \in [0,1]^{n^2}} \Lagr\left(D, \alpha, \hat{\lambda}, \hat{\tau}\right) + v,\]
and 
\[\Lagr(\hat{D}, \hat{\alpha}, \hat{\lambda}, \hat{\tau}) \ge \max_{\lambda \in \Lambda, \tau \in \Tau} \Lagr\left(\hat{D}, \hat{\alpha}, \lambda, \tau\right) - v.\]

Then, $err\left(\hat{D}, S\right) \le OPT + 2v$. And as for the constraints, $\E_{h\sim \hat{D}}\left[h(x_i) - h(x_j)\right] \le \hat{\alpha}_{ij} + \gamma + \frac{1+2v}{C_\lambda}, \forall (i,j) \in [n]^2$ and
$\frac{1}{|A|}\sum_{(i, j)\in [n]^2} \hat{w}_{ij} \hat{\alpha}_{ij}  \le \eta + \frac{1+2v}{C_\tau}$.
\end{theorem}
\input{approx-equilibrium-thm-proof}
\fi

\subsubsection{Omitted proof of theorem \ref{thm:algorithm-guarantee}}
\begin{proof}[proof of theorem \ref{thm:algorithm-guarantee}]
Observe that
\[
\Lagr(D, \alpha, \lambda, \tau) = err(D,S) + \Lagr_{D, \alpha}^{\psi_1}(\lambda) + \Lagr_{D, \alpha}^{\psi_2}(\tau)
\]

By how we constructed $\Lagr_{D, \alpha}^{\psi_1}$ and $\Lagr_{D, \alpha}^{\psi_2}$, combining Lemma \ref{lem:tau-no-regret} and \ref{lem:lambda-no-regret} yields

\begin{align*}
&\max_{\lambda \in \Lambda, \tau \in \Tau} \sum_{t=1}^T \Lagr\left(D^t, \alpha^t, \lambda^t, \tau^t\right) - \sum_{t=1}^T \Lagr\left(D^t, \alpha^t, \lambda^t, \tau^t\right)\\
&=\max_{\tau \in \Tau} \sum_{t=1}^T \Lagr_{D^t, \alpha^t}^{\psi_2}(\tau)  - \sum_{t=1}^T \Lagr_{D^t, \alpha^t}^{\psi_2}\left(\tau^t\right) + \max_{\lambda \in \Lambda} \sum_{t=1}^T \Lagr_{D^t, \alpha^t}^{\psi_1}(\lambda)  - \sum_{t=1}^T \Lagr_{D^t, \alpha^t}^{\psi_1}\left(\lambda^t\right)\\
&\le \xi_\psi T,
\end{align*}

where $\xi_\psi = \frac{2C_{\lambda}  \sqrt{T \log n} + C_{\tau}\sqrt{T}}{T}$.

Then, theorem \ref{thm:freund1996game} tells us that $\bar{D}, \bar{\alpha}, \bar{\lambda}, \bar{\alpha}$ form a $\xi_\psi$-approximate equilibrium, where $\bar{D} = \frac{1}{T}\sum_{t=1}^T D^t$, $\bar{\alpha} = \frac{1}{T}\sum_{t=1}^T \alpha^t$, $\bar{\lambda} = \frac{1}{T}\sum_{t=1}^T \lambda^t$, and $\bar{\tau} = \frac{1}{T} \sum_{t=1}^T \tau^t$. And finally, with $T= \left(\frac{2C_{\lambda}\sqrt{\log(n)} +C_{\tau}}{v}\right)^2$ results in $\xi_\psi = \nu$, theorem \ref{thm:approx-equilibrium} gives 
\[err(\hat{D}, S) \le \min_{(D,\alpha) \in \Omega(S, \hat{w}, \gamma, \eta)} \err(D, S) + 2\nu.\] 

 And as for the constraints, \[\E_{h\sim \hat{D}}\left[h(x_i) - h(x_j)\right] \le \hat{\alpha}_{ij} + \gamma + \frac{1+2\nu}{C_\lambda}, \forall (i,j) \in [n]^2\] and
\[\frac{1}{|A|}\sum_{(i, j)\in [n]^2} \hat{w}_{ij} \hat{\alpha}_{ij}  \le \eta + \frac{1+2v}{C_\tau}.\]
\end{proof}

\section{Generalization}
\subsubsection{Error}
\begin{theorem}[\cite{kearns1994introduction}]\label{thm:error-generalization}
Fix some hypothesis class $\mathcal{H}$ and distribution $\cP$. Let $S \sim P^n$ be a dataset consisting of $n$ examples $\{x_i, y_i\}_{i=1}^n$ sampled i.i.d. from $\cP$. Then, for any $0<\delta<1$, with probability $1-\delta$, for every $h \in \mathcal{H}$, we have
\[
\left\vert err(h, \cP) - err(h, S)\right\vert \le O\left( \sqrt{\frac{VCDIM(\mathcal{H})+ log(\frac{1}{\delta})}{n}}\right)
\]
\end{theorem}
\subsubsection{Fairness Loss}
At a high level, our argument proceeds as follows: using McDiarmid's inequality, for any \emph{fixed} hypothesis, its empirical fairness loss concentrates around its expectation. This argument extends to an infinite family of hypotheses with bounded VC-dimension via the standard two-sample trick, together with Sauer's lemma: the only catch is that we need to use a variant of McDiarmid's inequality that applies to sampling without replacement. However, proving that the fairness loss for each fixed hypothesis $h$ concentrates around its expectation is not sufficient to obtain the same result for arbitrary distributions over hypotheses, because the difference between a randomized classifier's fairness loss and its expectation is a non-convex function of the mixture weights. To circumvent this issue, we show that with respect to fairness loss, there is an $\epsilon$-net consisting of sparse distributions over hypotheses. Once we apply Sauer's lemma and the two-sample trick, there are only finitely many such distributions, and we can union bound over them. 

We begin by stating the standard version of McDiarmid's inequality:

\begin{theorem}[McDiarmid's Inequality]
Suppose $X_1, \dots, X_n$ are independent and $f$ satisfies
\[\sup_{x_1, \dots, x_n, \hat{x}_i} \left\vert f(x_1, \dots, x_n) - f(x_1, \dots, x_{i-1}, \hat{x}_i, x_{i+1}, \dots, x_n)\right\vert \le c_i.\]
Then, for any $\epsilon > 0$,
\[
\Pr_{X^1, \dots, X^n }\left(\left\vert f(X_1, \dots, X_n) - \E_{X_1, \dots, X_n}\left[f(X_1, \dots, X_n)\right]\right\vert \ge \epsilon \right)\le 2\exp\left(-\frac{2\epsilon^2}{\sum_{i=1}^n c_i^2}\right)
\]
\end{theorem}

\begin{lemma}\label{lem:sample_concentration}
Fix a randomized hypothesis $D \in \Delta\mathcal{H}$. Over the randomness of $S \sim \cP^n$, we have
\[
\Pr_{S \sim \cP^n}\left( \left\vert\Pi_{D, w, \gamma}(S \times S) - \E_{S}\left[\Pi_{D, w, \gamma}(S \times S)\right]\right\vert \ge \epsilon  \right)\le  2\exp\left(-2n\epsilon^2\right)
\]
\end{lemma}
\begin{proof}
Define a slightly modified fairness loss function that depends on each instance instead of a pair.
\[
\Pi'_{D,w,\gamma}\left(x_1, x_2, \dots, x_{n}\right) = \frac{1}{{n}^2}\sum_{(i,j) \in [n]^2} \Pi_{D,w,\gamma}\left((x_i, x_j)\right).
\]
Note that $\Pi'_{D,w,\gamma}(x_1, \dots, x_n) = \Pi_{D,w,\gamma}(S \times S)$. The sensitivity of $\Pi'_{D,w,\gamma}(x_1, x_2, \dots, x_{n})$ is $\frac{1}{n}$, so applying McDiarmid's inequality yields the above concentration.
\end{proof}

\begin{theorem}\label{thm:fairness-loss-all-pairs}
If $n \ge \frac{2\ln(2)}{\epsilon^2}$,
\[
\Pr_{S}\left(\sup_{D \in \Delta\mathcal{H}}\left\vert\Pi_{D,w,\gamma}(S \times S) - \E_{x,x'}\left[\Pi_{D,w,\gamma}(x,x')\right]\right\vert > \epsilon \right) \le 8 \cdot \left(\frac{e\cdot 2n}{d}\right)^{dk} \exp\left(\frac{-n\epsilon^2}{32}\right)
\]
where $d$ is the VC-dimension of $\mathcal{H}$, and $k=\frac{\ln(2n^2)}{8\epsilon^2} + 1$.
\end{theorem}
\begin{proof}
%This proof is a slightly modified version of the proof for VC-inequality; it uses McDiarmid's inequality instead of Hoeffding's in order to circumvent dependency between pairs and also incorporates a net argument in order to argue for uniform convergence over $\Delta\mathcal{H}$ and not just over $\mathcal{H}$.
First, by linearity of expectation, we note that $\E_{S}\left[\Pi_{D,w,\gamma}(S \times S)\right]  = \E_{x,x'}\left[\Pi_{D,w,\gamma}(x,x')\right]$. Given $S$, let $D^*_S$ be some randomized classifier such that $\left\vert\Pi_{D^*_S,w,\gamma}(S \times S) - \E_{x,x'}\left[\Pi_{D^*_S,w,\gamma}(x,x')\right]\right\vert > \epsilon$; if such hypothesis does not exist, let it be some fixed hypothesis in $\mathcal{H}$. We now use standard symmetrization argument, which allows us to bound the difference between the fairness loss of our sample $S$ and that of another independent `ghost' sample $S' = (x'_1, \dots, x'_n)$ instead of bounding the difference between the empirical fairness loss and its expected fairness loss.
\begin{align*}
&\Pr_{S \sim \cP^n,S' \sim \cP^n} \left(\sup_{D \in \Delta\mathcal{H}}\left\vert\Pi_{D,w,\gamma}(S \times S) - \Pi_{D,w,\gamma}(S' \times S') \right\vert> \frac{\epsilon}{2} \right)\\
&\ge\Pr_{S,S'} \left(\left\vert \Pi_{D^*_S,w,\gamma}(S \times S) - \Pi_{D^*_S,w,\gamma}(S' \times S')\right\vert > \frac{\epsilon}{2} \right)\\
&\ge \Pr_{S, S'} \left(\left\vert \Pi_{D^*_S,w,\gamma}(S \times S) - \E_{x,x'}\left[\Pi_{D^*_S,w,\gamma}(x,x')\right] \right\vert > \epsilon \text{ and } \left\vert \Pi_{D^*,w,\gamma}(S' \times S') - \E_{x,x'}\left[\Pi_{D^*,w,\gamma}(x,x')\right]\right\vert \le \frac{\epsilon}{2} \right) \\
&= \E_{S,S'}\left[ \ind\left(\left\vert \Pi_{D^*_S,w,\gamma}(S \times S) - \E_{x,x'}\left[\Pi_{D^*_S,w,\gamma}(x,x')\right]\right\vert > \epsilon\right) \cdot \ind\left(\left\vert\Pi_{D^*,w,\gamma}(S' \times S') - \E_{x,x'}\left[\Pi_{D^*,w,\gamma}(x,x')\right]\right\vert \le \frac{\epsilon}{2}\right) \right]\\
&= \E_{S}\left[ \ind\left(\left\vert \Pi_{D^*_S,w,\gamma}(S \times S) - \E_{x,x'}\left[\Pi_{D^*_S,w,\gamma}(x,x')\right]\right\vert > \epsilon\right) \cdot \Pr_{S' \vert S}\left(\left\vert\Pi_{D^*,w,\gamma}(S' \times S') - \E_{x,x'}\left[\Pi_{D^*,w,\gamma}(x,x')\right]\right\vert \le \frac{\epsilon}{2}\right) \right]\\
&\ge \Pr_{S}\left(\left\vert\Pi_{D^*_S,w,\gamma}(S \times S) - \E_{x,x'}\left[\Pi_{D^*_S,w,\gamma}(x,x')\right]\right\vert > \epsilon) \right) \cdot \left(1-\exp(-\frac{n\epsilon^2}{2}) \right)\\
&\ge \frac{1}{2} \Pr_{S}\left(\sup_{D \in \Delta\mathcal{H}}\left\vert\Pi_{D,w,\gamma}(S \times S) - \E_{x,x'}\left[\Pi_{D,w,\gamma}(x,x')\right]\right\vert > \epsilon \right)\\
\end{align*}

We used Lemma \ref{lem:sample_concentration} for the second to last inequality, and the last inequality follows from the theorem's condition and the definition of $D^*_S$.

Now, imagine sampling $\bar{S} = 2n$ points from $\cP$, and uniformly choosing $n$ points without replacement to be $S$ and the remaining $n$ points to be $S'$. This process is equivalent to sampling $n$ points from $\cP$ to form $S$ and another independent set of $n$ points from $\cP$ to form $S'$.
\begin{align*}
&\Pr_{\bar{S}, S,S'}\left(\sup_{D \in \Delta\mathcal{H}}\left\vert\Pi_{D,w,\gamma}(S \times S) - \Pi_{D,w,\gamma}(S' \times S') \right\vert > \frac{\epsilon}{2}\right) \\
&= \sum_{\bar{S}} \Pr\left(\bar{S}\right) \Pr_{S,S'}\left(\sup_{D \in \Delta\mathcal{H}}\left\vert\Pi_{D,w,\gamma}(S \times S) - \Pi_{D,w,\gamma}(S' \times S') \right\vert > \frac{\epsilon}{2} \Bigg| \bar{S}\right) \\
\end{align*}

Now, instead of bounding the supremum over $\Delta\mathcal{H}$, we pay approximation error of $\epsilon'$ in order to bound the supremum over $\mathcal{H}$.

\begin{lemma}\label{lem:approximate_random}
For some fixed data sample $S$ of size $n$, any $D \in \Delta\mathcal{H}$ can be approximated by some uniform mixture over $k := \frac{2 \ln(2n^2)}{\epsilon'^2} + 1$ hypotheses  $\hat{D} = \frac{1}{k} \{h_1, \dots, h_{k}\}$ such that for every $(x,x') \in S \times S$,
\[\left\vert \E_{h \sim D}\left[h(x) - h(x')\right] - \E_{h \sim \hat{D}}\left[h(x) - h(x')\right] \right\vert \le \epsilon'. \]
\end{lemma}
\begin{proof}
Fix some $(x,x') \in S \times S$. Randomly sample $k$ hypotheses from $D$: $\{h_i\}_{i=1}^k \sim D^k$. Because for each randomly drawn hypothesis $h_i \sim D$, the difference in its prediction for $x$ and $x'$ is exactly $\E_{h \sim D}[h(x) - h(x')]$, Hoeffding's inequality yields that

\[
\Pr_{h_i \sim D, i \in [k]}\left(\left\vert \E_{h \sim D}\left[h(x) - h(x')\right] - \frac{1}{k} \sum_{i=1}^k \left[h_i(x) - h_i(x')\right] \right\vert > \epsilon'\right) \le 2\exp\left(-\frac{2k^2\epsilon'^2}{4k}\right) = 2\exp\left(-\frac{k\epsilon'^2}{2}\right).
\]

However, there are $n^2$ fixed pairs in $S \times S$, and if we distribute the failure property between $n^2$ pairs and union bound over all of them, we get
\[
\Pr_{h_i \sim D, i \in [k]} \left( \max_{(x,x') \in S \times S} \left\vert \E_{h \sim D}\left[h(x) - h(x')\right] - \frac{1}{k} \sum_{i=1}^k [h_i(x) - h_i(x')] \right\vert > \epsilon' \right) \le 2n^2\exp\left(-\frac{k\epsilon'^2}{2}\right).
\]
In order to achieve non-zero probability of having \[\left\vert \E_{h \sim D}\left[h(x) - h(x')\right] - \frac{1}{k} \sum_{i=1}^k [h_i(x) - h_i(x')] \right\vert \le \epsilon', \forall (x,x') \in S \times S,\] we need to make sure $2n^2\exp\left(-\frac{k\epsilon'^2}{2}\right) < 1$ or $k > \frac{2 \ln\left(2n^2\right)}{\epsilon'^2}$.

\end{proof}

\begin{corollary}\label{cor:approximate-random}
For some fixed data sample $S$, any $D \in \Delta\mathcal{H}$ can be approximated by a uniform mixture of $k := \frac{2 \ln(2n^2)}{\epsilon'^2} + 1$ hypotheses  $\hat{D} = \frac{1}{k} \{h_1, \dots, h_{k}\}$ such that
\[
\left\vert \Pi_{D, w, \gamma}(S \times S) -  \Pi_{\hat{D}, w, \gamma}(S \times S)\right\vert \le \epsilon'
\]
\end{corollary}
\begin{proof}
It simply follows from Lemma \ref{lem:approximate_random} and the fact that $\max\left(0, \E_{h \sim D}\left[h(x_i) - h(x_j)\right] - \gamma \right)$ is 1-Lipschitz in terms of $\E_{h \sim D}[h(x_i) - h(x_j)]$.
\end{proof}

Using Corollary \ref{cor:approximate-random} and using Sauer's lemma that bounds the total number of possible labelings by $\mathcal{H}$ over $2n$ points to be $\left(\frac{e\cdot 2n}{d}\right)^d$, we can show
\begin{align*}
&\sum_{\bar{S}} \Pr\left(\bar{S}\right) \Pr_{S,S'}\left(\sup_{D \in \Delta\mathcal{H}}\left\vert\Pi_{D,w,\gamma}(S \times S) - \Pi_{D,w,\gamma}(S' \times S') \right\vert > \frac{\epsilon}{2}  \bmid \bar{S}\right) \\
&\le \sum_{\bar{S}} \Pr\left(\bar{S}\right) \Pr_{S,S'}\left(\sup_{\hat{D} \in \mathcal{H}^k}\left\vert\Pi_{\hat{D},w,\gamma}(S \times S) - \Pi_{\hat{D},w,\gamma}(S' \times S') \right\vert > \frac{\epsilon}{2} + \epsilon' \bmid \bar{S}\right) \\
&\le \sum_{\bar{S}} \Pr\left(\bar{S}\right) \cdot \left(\frac{e\cdot 2n}{d}\right)^{dk} \sup_{\hat{D} \in \mathcal{H}^k} \Pr_{S,S'}\left(\left\vert\Pi_{\hat{D},w,\gamma}(S \times S) - \Pi_{\hat{D},w,\gamma}(S' \times S') \right\vert > \frac{\epsilon}{2} + \epsilon' \bmid \bar{S}\right) \\
\end{align*}

Now, for any $\hat{D}$, we will try to bound the probability that the difference in fairness loss between $S$ and $S'$ is big. We do so by union bounding over cases where both of them deviate from its mean by too much.

If $\left\vert\Pi_{\hat{D},w,\gamma}(S \times S) - E_{S | \bar{S}}\left[\Pi_{\hat{D},w,\gamma}(S \times S)\right] \right\vert \le \frac{\epsilon}{4} + \frac{\epsilon'}{2}$ and $\left\vert\Pi_{\hat{D},w,\gamma}(S' \times S') - E_{S | \bar{S}}\left[\Pi_{\hat{D},w,\gamma}(S \times S)\right] \right\vert \le \frac{\epsilon}{4} + \frac{\epsilon'}{2}$, then $\left\vert\Pi_{\hat{D},w,\gamma}(S \times S) - \Pi_{\hat{D},w,\gamma}(S' \times S') \right\vert \le \frac{\epsilon}{2} + \epsilon'$. In other words,

\begin{align*}
&\Pr_{S,S'}\left(\left\vert\Pi_{\hat{D},w,\gamma}(S \times S) - \Pi_{\hat{D},w,\gamma}(S' \times S') \right\vert \le \frac{\epsilon}{2} + \epsilon' \bmid \bar{S}\right)\\
&\ge \Pr_{S, S'}\left(\left\vert\Pi_{\hat{D},w,\gamma}(S \times S) - E_{S | \bar{S}}\left[\Pi_{\hat{D},w,\gamma}(S \times S)\right] \right\vert \le \frac{\epsilon}{4} + \frac{\epsilon'}{2} \text{ and }  \left\vert\Pi_{\hat{D},w,\gamma}(S' \times S') - E_{S | \bar{S}}\left[\Pi_{\hat{D},w,\gamma}(S \times S)\right] \right\vert \le \frac{\epsilon}{4} + \frac{\epsilon'}{2} \bmid \bar{S}\right).
\end{align*}

Therefore, by looking at the compliment probabilities, we have
\begin{align*}
&\Pr_{S,S'}\left(\left\vert\Pi_{\hat{D},w,\gamma}(S \times S) - \Pi_{\hat{D},w,\gamma}(S' \times S') \right\vert > \frac{\epsilon}{2} + \epsilon' \bmid \bar{S}\right) \\
&\le \Pr_{S, S'}\left(\left\vert\Pi_{\hat{D},w,\gamma}(S \times S) - E_{S | \bar{S}}\left[\Pi_{\hat{D},w,\gamma}(S \times S)\right] \right\vert > \frac{\epsilon}{4} + \frac{\epsilon'}{2} \text{ or }  \left\vert\Pi_{\hat{D},w,\gamma}(S' \times S') - E_{S | \bar{S}}\left[\Pi_{\hat{D},w,\gamma}(S \times S)\right] \right\vert > \frac{\epsilon}{4} + \frac{\epsilon'}{2} \bmid \bar{S}\right)\\
&\le 2\Pr_{S}\left(\left\vert\Pi_{\hat{D},w,\gamma}(S \times S) - E_{S | \bar{S}}\left[\Pi_{\hat{D},w,\gamma}(S \times S)\right] \right\vert > \frac{\epsilon}{4} + \frac{\epsilon'}{2} \bmid \bar{S}\right).
\end{align*}

Here, we can't appeal to McDiarmid's because $S$ is sampled without replacement from $\bar{S}$. However, we can use the same technique that \cite{neel2018use} leveraged -- stochastic covering property can be used to show concentration for sampling without replacement \citep{pemantle2014concentration}. 

\begin{definition}[\cite{pemantle2014concentration}]
$Z_1, \dots, Z_n$ satisfy the stochastic covering property, if for any $I \subset [n]$ and $a \ge a' \in \{0,1\}^I$ coordinate-wise such that $||a' - a||_1 = 1$, there is a coupling $\nu$ of the distributions $\mu, \mu'$ of $(Z_j: j \in [n] \setminus I)$ conditioned on $Z_I = a$ or $Z_I = a'$, respectively, such that $\nu(x, y) = 0$ unless $x \le y$ coordinate-wise and $||x-y||_1 \le 1$. 
\end{definition}

\begin{theorem}[\cite{pemantle2014concentration}]\label{thm:without-replacement-concentration} 
Let $(Z_1, \dots, Z_n) \in \{0,1\}$ be random variables such that $\Pr(\sum_{i=1}^n Z_i = k)=1$  and the stochastic covering property is satisfied. Let $f: \{0,1\}^n \to \cR$ be an $c$-Lipschitz function. Then, for any $\epsilon > 0$,
\[
\Pr\left(\left\vert f(Z_1, \dots, Z_n) - \E\left[f(Z_1, \dots, Z_n)\right]\right\vert \ge \epsilon\right) \le 2\exp\left(\frac{-\epsilon^2}{8c^2k}\right)
\]
\end{theorem}

\begin{lemma}[\cite{neel2018use}]\label{lem:without-replacement-property}
Given a set $S$ of $n$ points, sample $k \le n$ elements without replacement. Let $Z_i = \{0,1\}$ indicate whether $i$th element has been chosen. Then, $(Z_1, \dots, Z_n)$ satisfy the stochastic covering property.
\end{lemma}

Let $\bar{S} = \{x_1, \dots, x_{2n}\}$. If we slightly change the definition of the fairness loss so that it depends on the indicator variables $Z_1, \dots, Z_{2n}$,
\[\Pi''_{\hat{D}, w, \gamma, \bar{S}}(Z_1, \dots, Z_{2n})= \frac{1}{n^2} \sum_{i,j \in [2n]^2} Z_i Z_j \Pi_{\hat{D}, w, \gamma}(x_i, x_j) = \Pi_{\hat{D}, w, \gamma}(S \times S).\]

We see that $\Pi''_{\hat{D}, w, \gamma, \bar{S}}$ is $\frac{1}{n}$-Lipschitz, so by theorem \ref{thm:without-replacement-concentration} and lemma \ref{lem:without-replacement-property}, we get

\[
\Pr_{S}\left(\left\vert\Pi_{\hat{D},w,\gamma}(S \times S) - E_{S | \bar{S}}[\Pi_{\hat{D},w,\gamma}(S \times S)] \right\vert > \frac{\epsilon}{4} + \frac{\epsilon'}{2} \bmid \bar{S}\right) \le 2\exp\left(\frac{-\left(\frac{\epsilon}{4} + \frac{\epsilon'}{2}\right)^2}{8\frac{1}{n^2} \cdot n}\right) = 2\exp\left(\frac{-n\left(\frac{\epsilon}{4} + \frac{\epsilon'}{2}\right)^2}{8}\right)
\]

Combining everything, we get
\begin{align*}
&\Pr_{S}\left(\sup_{D \in \Delta\mathcal{H}}\left\vert\Pi_{D,w,\gamma}(S \times S) - \E_{x,x'}[\Pi_{D,w,\gamma}(x,x')]\right\vert > \epsilon \right) \\
&\le 2\sum_{\bar{S}} \Pr\left(\bar{S}\right) \cdot \left(\frac{e\cdot 2n}{d}\right)^{dk} \sup_{\hat{D} \in \mathcal{H}^k} \Pr_{S,S'}\left(\left\vert\Pi_{\hat{D},w,\gamma}(S \times S) - \Pi_{\hat{D},w,\gamma}(S' \times S') \right\vert > \frac{\epsilon}{2} + \epsilon' \bmid \bar{S}\right)\\
&\le 4\sum_{\bar{S}} \Pr\left(\bar{S}\right) \cdot \left(\frac{e\cdot 2n}{d}\right)^{dk} \sup_{\hat{D} \in \mathcal{H}^k} \Pr_{S}\left(\left\vert\Pi_{\hat{D},w,\gamma}(S \times S) - E_{S | \bar{S}}\left[\Pi_{\hat{D},w,\gamma}(S \times S)\right] \right\vert > \frac{\epsilon}{4} + \frac{\epsilon'}{2} \bmid \bar{S}\right) \\
&\le 8 \cdot \left(\frac{e\cdot 2n}{d}\right)^{dk} \exp\left(\frac{-n\left(\frac{\epsilon}{4} + \frac{\epsilon'}{2}\right)^2}{8}\right)
\end{align*}

For convenience, we set $\epsilon'=\frac{\epsilon}{2}$.
%\cj{I might the probability as it is, and instead calculate sample complexity in the final theorem statement}
%For convenience, setting $\epsilon' = \frac{\epsilon}{2}$ and the failure probability to equal $\delta$, we have
%\begin{align*}
%\delta &= 8 \cdot (\frac{e\cdot 2n}{d})^{dk} \exp(\frac{-n\epsilon^2}{32})\\
%\ln(\delta) &= \ln(8) + dk \ln(\frac{e\cdot 2n}{d}) -\frac{n\epsilon^2}{32}\\
%\frac{n\epsilon^2}{32} &= \ln(8) + d (\frac{8 \ln(2n^2)}{\epsilon^2} + 1) (1 + \ln(2) + \ln(n) - \ln(d)) -\ln(\delta)
%\end{align*}

\end{proof}

However, in our case, instead of finding the average over all pairs in $S$, we calculate the fairness loss only over $m$ pairs. Fixing $S$, if $m$ is sufficiently large, our empirical fairness loss should concentrate around the fairness loss over all the pairs for $S$.

\begin{lemma}\label{lem:random-pairs}
For fixed $S$, randomly chosen pairs $M \subset S \times S$, and randomized hypothesis $D$,
\[\Pr_{M \sim (S \times S)^m}\left( \Pi_{D,w,\gamma}(M) - \Pi_{D,w,\gamma}(S \times S) \ge \epsilon \right) \le \exp\left(-2m\epsilon^2\right)\]
\end{lemma}
\begin{proof}
Write a random variable $L_a = \Pi_{D,w,\gamma}((x_{2a-1}, x_{2a}))$ for the fairness loss of the $a$th pair.
Note that \[E[L_a] = \sum_{(i,j) \in [n]^2} \frac{1}{n^2} \Pi_{D,w,\gamma}\left((x_i,x_j)\right)=\Pi_{D,w,\gamma}(S \times S), \forall a \in [|M|].\]
Therefore, by Hoeffding's inequality, we have
\[
\Pr_{M}\left( \Pi_{D,w,\gamma}(M) - \Pi_{D,w,\gamma}(S \times S) \ge \epsilon \right) \le \exp\left(-2m\epsilon^2\right).
\]
\end{proof}
\begin{lemma}\label{lem:unif-conver-random-pairs}
For fixed $S$ and randomly chosen pairs $M \subset S \times S$,
\[\Pr_{M \sim (S \times S)^m}\left(\sup_{D \in \Delta\mathcal{H}} \left\vert \Pi_{D,w,\gamma}(M) - \Pi_{D,w,\gamma}(S \times S) \right\vert \ge \epsilon \right) \le \left(\frac{e\cdot 2n}{d}\right)^{dk'} \exp\left(-8m\epsilon^2\right),\]
where $k'=\frac{2 \ln(2m)}{\epsilon^2} + 1$.
\end{lemma}
\begin{proof}
\begin{align*}
&\Pr_{M \sim (S \times S)^m}\left(\sup_{D \in \Delta\mathcal{H}} \left\vert \Pi_{D,w,\gamma}(M) - \Pi_{D,w,\gamma}(S \times S) \right\vert \ge \epsilon \right)\\
&\le \Pr_{M \sim (S \times S)^m}\left(\sup_{\hat{D} \in \mathcal{H}^k} \left\vert \Pi_{\hat{D},w,\gamma}(M) - \Pi_{\hat{D},w,\gamma}(S \times S) \right\vert \ge \epsilon + 2\epsilon' \right)\\
&\le \sum_{\hat{D} \in \mathcal{H}^k} \Pr_{M \sim (S \times S)^m}\left( \left\vert \Pi_{\hat{D},w,\gamma}(M) - \Pi_{\hat{D},w,\gamma}(S \times S) \right\vert \ge \epsilon + 2\epsilon' \right)\\
&\le  \left(\frac{e\cdot 2n}{d}\right)^{dk} \exp\left(-2m\left(\epsilon + 2\epsilon'\right)^2\right),
\end{align*}
where $k = \frac{2 \ln(2m)}{4\epsilon'^2} + 1$. The last inequality is from Corollary \ref{cor:approximate-random} and Lemma \ref{lem:random-pairs}. For convenience, we just set $\epsilon' = \epsilon/2$.
\end{proof}

\subsection{Omitted proof of theorem \ref{thm:fairness-loss-generalization}}
Combining theorem \ref{thm:fairness-loss-all-pairs} and lemma \ref{lem:unif-conver-random-pairs} yields the following theorem for fairness loss generalization.
\begin{proof}[proof of theorem \ref{thm:fairness-loss-generalization}]
With probability $1-\left( 8 \cdot (\frac{e\cdot 2n}{d})^{dk} \exp\left(\frac{-n\epsilon^2}{32}\right)+ \left(\frac{e\cdot 2n}{d}\right)^{dk'} \exp\left(-8m\epsilon^2\right) \right)$, where $k'=\frac{2 \ln(2m)}{\epsilon^2} + 1$ and $k=\frac{\ln(2n^2)}{8\epsilon^2} + 1$,
we have
\[\sup_{D \in \Delta\mathcal{H}} \left\vert \Pi_{D,w,\gamma}(M) - \Pi_{D,w,\gamma}(S \times S) \right\vert \le \epsilon\]
and
\[\sup_{D \in \Delta\mathcal{H}} \left\vert \Pi_{D,w,\gamma}(S \times S) - \E_{x,x'}[\Pi_{D,w,\gamma}(x,x') \right\vert \le \epsilon.\]

Then, by triangle inequality,
\begin{align*}
\sup_{D \in \Delta\mathcal{H}} \left\vert \Pi_{D,w,\gamma}(M) - \E_{x,x'}[\Pi_{D,w,\gamma}(x,x') \right\vert \le 2\epsilon.
\end{align*}

In other words, with probability $\left( 8 \cdot \left(\frac{e\cdot 2n}{d}\right)^{dk} \exp\left(\frac{-n\epsilon^2}{32}\right)+ \left(\frac{e\cdot 2n}{d}\right)^{dk'} \exp\left(-8m\epsilon^2\right) \right)$, we have
\[
\sup_{D \in \Delta\mathcal{H}} \left\vert \Pi_{D,w,\gamma}(M) - \E_{x,x'}\left[\Pi_{D,w,\gamma}(x,x')\right] \right\vert > 2\epsilon.
\]

\end{proof}

%%% Local Variables:
%%% mode: latex
%%% TeX-master: "main"
%%% End:

\end{document}